\documentclass{article}
\PassOptionsToPackage{numbers, compress}{natbib}
\usepackage[preprint]{neurips_2026}

\usepackage[utf8]{inputenc} % allow utf-8 input
\usepackage[T1]{fontenc} % use 8-bit T1 fonts
\usepackage{hyperref} % hyperlinks
\usepackage{url} % simple URL typesetting
\usepackage{times}
\usepackage{booktabs, multirow} % tables
\usepackage{graphicx, subcaption, wrapfig} % figures
\usepackage{algorithm, algorithmic}
\usepackage{latexsym, amsmath, amsthm, amssymb, amsfonts, bm} % amsthm: proof, amsfonts: mathbb, bm: bold.
\usepackage{nicefrac} % compact symbols for 1/2, etc.
\usepackage{tcolorbox}
\usepackage{xcolor} % colors
\usepackage{enumitem}
\usepackage{caption}
\usepackage{cleveref} % for reference
\usepackage{microtype} % microtypography
\usepackage{comment}

\setlist[enumerate]{itemsep=0mm}
\newtheorem{definition}{Definition}
\newtheorem{lemma}{Lemma}
\newtheorem{theorem}{Theorem}
\newcommand{\dnns}[0]{\textsc{dnn}s}
\newcommand{\qnn}[0]{\textsc{qnn}}
\newcommand{\qnns}[0]{\textsc{qnn}s}
\newcommand{\vqcs}[0]{\textsc{vqc}s}

\title{SLT: Robust Quantum Neural Networks for Noisy-Label Medical Image Classification via Supermartingale-based Label Transition}
\author{Jun Zhuang\textsuperscript{1} \quad
Mohammad Al Hasan\textsuperscript{2} \quad
Yiyu Shi\textsuperscript{3} \quad
Chaowen Guan\textsuperscript{4}\\
\textsuperscript{1}Boise State University, ID 83725 \quad
\textsuperscript{2}Indiana University Indianapolis, IN 46202 \quad\\
\textsuperscript{3}University of Notre Dame, IN 46556 \quad
\textsuperscript{4}University of Cincinnati, OH 45221\\
\textsuperscript{1}\texttt{junzhuang@boisestate.edu} \quad
\textsuperscript{2}\texttt{alhasan@iu.edu} \quad\\
\textsuperscript{3}\texttt{yshi4@nd.edu} \quad
\textsuperscript{4}\texttt{guance@ucmail.uc.edu}
}

\begin{document}
\maketitle

\begin{abstract}
Noisy-label learning in small-scale medical image classification remains challenging, as limited data can amplify annotation errors. While quantum neural networks (\qnns) have recently been explored as compact models for small-data regimes, their behavior under label corruption remains insufficiently understood.
A key obstacle is \qnns' intrinsic ``natural smoothness'', which may regularize training but also obscure high-confidence samples needed for noise-transition estimation. We propose Supermartingale-based Label Transition (SLT), an anchor-free loss correction framework for robust \qnn-based medical image classification under noisy labels. SLT models entropy reduction in predictive distributions as a supermartingale and uses its monotonic behavior to identify stable transition-matrix refinement steps. This enables dynamic transition updates while reducing noise-driven oscillations during \qnn\ training. We further provide a convergence analysis showing that the proposed transition-refinement process reaches a steady state. Experiments on multiple public small-scale medical image datasets demonstrate that SLT improves \qnn-based noisy-label classification and outperforms classic noise-label learning baselines under synthetic and real-world label noise.
\end{abstract}

\section{Introduction}
\label{sec:intro}
Noisy-label learning in small-scale medical image classification remains a central challenge for reliable medical applications due to costly expert annotations~\cite{shi2024survey}. Classical deep neural networks (\dnns) have achieved remarkable success when trained on large-scale datasets with high-quality annotations~\cite{hestness2017deep}; however, their advantages often diminish when only limited training samples with corrupted labels are available~\cite{song2022learning}. This motivates alternative modeling paradigms suited to limited-data regimes. Recent studies suggest that quantum neural networks (\qnns) may offer potential benefits in small-data scenarios or tasks involving complex, high-dimensional feature spaces~\cite{havlivcek2019supervised, maqsudur2025nqnn, oviesi2025quantum}. Nevertheless, applying \qnns\ to noisy-label medical image classification remains under-explored.
A key challenge arises from the intrinsic ``natural smoothness'' of \qnns, which is rooted in the Born rule~\cite{born1926quantenmechanik}. Unlike classical \dnns, whose predictive distributions can rapidly become sharply peaked, \qnns\ tend to produce smoother probability distributions through quantum measurements. This property can serve as an implicit regularizer in limited-data settings, but it also makes noisy-label learning more difficult: overly smooth predictions may obscure reliable class-conditional patterns and hinder the identification of high-confidence samples for noise-transition estimation. Thus, a central question remains unresolved: \textit{can the smooth predictive behavior of \qnns\ be leveraged, rather than merely treated as a limitation, to improve robustness under noisy labels}?

To address noisy labeling issues, a line of research in classical machine learning has explored loss correction methods, which correct the training objective using a transition matrix that models the mapping between latent clean labels and observed noisy labels~\cite{patrini2017making}. The effectiveness of these methods, however, heavily depends on accurate transition matrix estimation~\cite{northcutt2021confident, kye2022learning, zhu2022beyond, liu2023identifiability}. A widely adopted strategy estimates the transition matrix using anchor points, i.e., instances whose posterior distribution is nearly one-hot for a specific class, such that the class-conditional noise distribution can be directly inferred~\cite{liu2015classification, patrini2017making}. Unfortunately, true anchor points are often unreliable or even unavailable in small-scale and severe label corruption medical datasets~\cite{karimi2020deep, ju2022improving}.
To mitigate this limitation, anchor-free methods have attempted to estimate the transition matrix without relying on explicit anchor points~\cite{yao2020dual, li2021provably, zhu2021clusterability, cheng2022class}. One representative strategy is to select high-confidence instances, also known as weak anchor points, whose predictive probabilities have low entropy and can serve as proxies for transition estimation~\cite{xia2019anchor, zhang2021learning}. While effective in some settings, such entropy-based selection can be fragile for classical \dnns\ due to overconfidence~\cite{wei2022mitigating}, where mislabeled samples may be selected as weak anchors, leading to biased transition estimates. In contrast, the ``natural smoothness'' of \qnns\ may reduce premature overconfidence, suggesting a promising opportunity for more stable weak-anchor selection. However, this opportunity requires a principled mechanism to determine when the transition matrix should be updated during training.

\begin{wrapfigure}{r}{0.45\textwidth}
\vspace{-6mm}
%\begin{figure}[t]
  \centering
  \begin{subfigure}[t]{0.32\linewidth}
    \centering
    \caption{No Noise}
    \includegraphics[width=\linewidth]{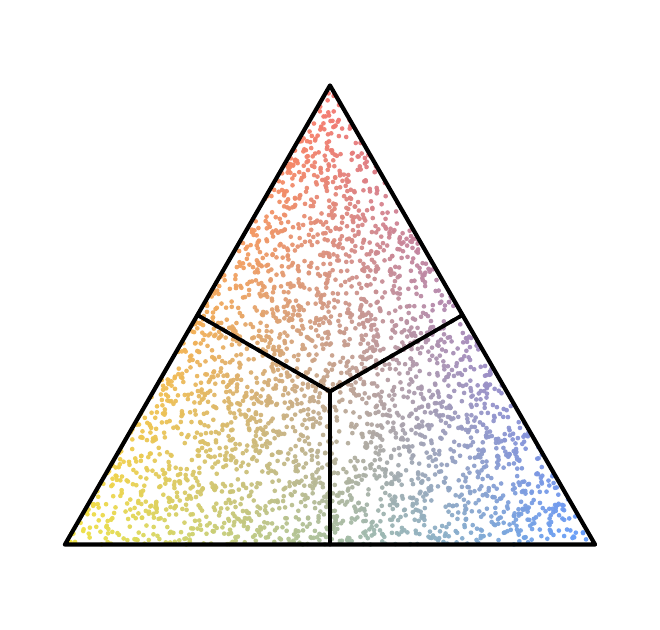}
  \end{subfigure}%\vspace{-2mm}
  \begin{subfigure}[t]{0.32\linewidth}
    \centering
    \caption{Symmetric}
    \includegraphics[width=\linewidth]{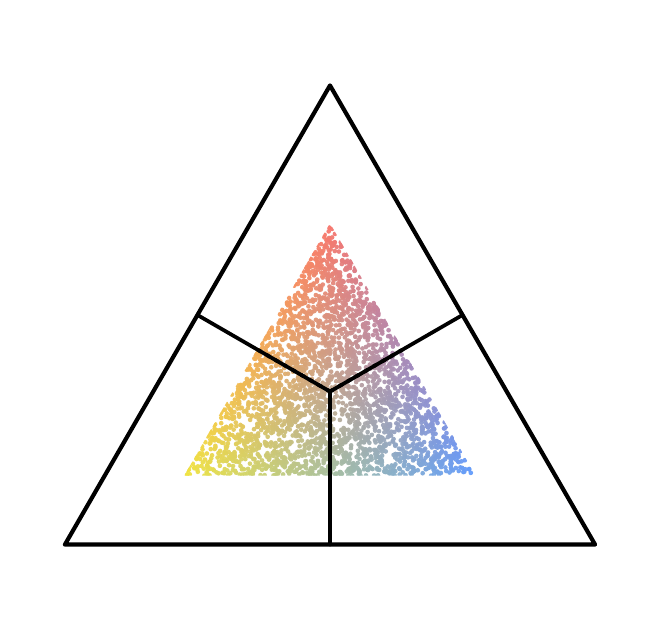}
  \end{subfigure}%\vspace{-2mm}
  \begin{subfigure}[t]{0.32\linewidth}
    \centering
    \caption{Asymmetric}
    \includegraphics[width=\linewidth]{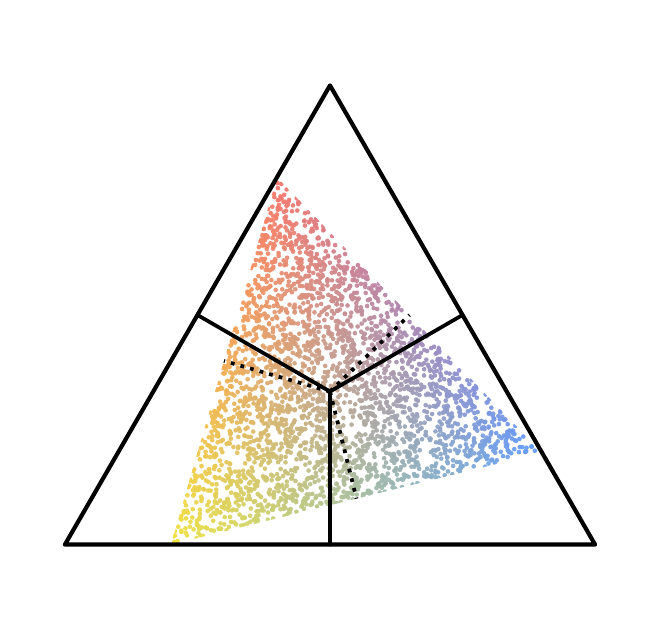}
  \end{subfigure}\vspace{-3mm}
  \caption{Example probability simplex under different label-noise settings. To construct noisy labels, we flip each clean label uniformly to any other class in the symmetric case (UN), or to a fixed alternative class in a cyclic/custom manner for asymmetric cases, including cyclic-flipping (CF) and custom-mapping (CM) noise.}
  \label{fig:simplex}
%\end{figure}
\vspace{-2mm}
\end{wrapfigure}
Inspired by this observation, we propose to model the entropy-reduction process of predictive distributions as a supermartingale. Rather than updating the transition matrix whenever the model becomes locally confident, our approach uses the monotonic behavior of a supermartingale to identify stable refinement steps at which predictions exhibit progressively lower entropy. This allows the transition matrix to be updated dynamically while reducing noise-driven oscillations during \qnn\ training. Furthermore, we theoretically prove that the proposed supermartingale-based transition process converges to a steady state, providing a stability criterion for label-transition refinement. Building on this theoretical insight, we introduce a Supermartingale-based Label Transition (SLT) framework that adaptively exploits high-confidence predictions to refine the transition matrix without relying on hard-to-find anchor points. Importantly, SLT turns the smooth predictive behavior of \qnns\ from a potential limitation into a stabilizing inductive bias: it delays unreliable transition updates, mitigates premature overconfidence in weak-anchor selection, and improves robustness under noisy annotations.
Extensive experiments across multiple public small-scale medical image datasets demonstrate that our framework improves \qnn-based classification and outperforms classic noise-label learning baselines under various synthetic and real-world noise settings, where we illustrate symmetric and asymmetric noise as an example in Fig.~\ref{fig:simplex}.
In brief, our primary contribution can be summarized as:
\begin{itemize}
    \item We propose SLT, a new anchor-free loss correction framework for \qnns\ that adaptively exploits high-confidence predictions to refine the noise transition matrix without relying on hard-to-find anchor points, making it suitable for small-scale medical image datasets with noisy annotations.
    \item We formalize the entropy-reduction process of predictive probabilities as a \emph{supermartingale} and theoretically analyze the convergence of the proposed transition-refinement process, providing a stability-aware foundation for label-noise correction.
    \item Extensive experiments on multiple public small-scale medical image datasets demonstrate that SLT improves \qnn-based classification under various label noise settings, while outperforming several representative baselines.
\end{itemize}

\section{Preliminaries}
\label{sec:pre}
In this section, we first introduce the preliminary background about the basics of quantum computing and barren plateaus, and then present the necessary tools from probability theory.

\noindent \textbf{Learning with Noisy Labels in \qnns.}
In this work, we aim to train \qnns\ $f_{\bm{\theta}}(\cdot)$ parameterized by $\bm{\theta}$ on a noisily annotated dataset $\{ \mathbf{X}, \mathbf{\tilde{Y}} \}$, where $\mathbf{X}\!\in\!\mathbb{R}^{|X|\times d}$ denotes the input matrix of $|X|$ instances with $d$-dimensional features, and $\mathbf{\tilde{Y}}\!\in\!\mathbb{R}^{|X| \times 1}$ denotes the observed {\bf noisy labels}, which may deviate from the latent {\bf ground-truth labels} $\mathbf{Y}\!\in\!\mathbb{R}^{|X| \times 1}$ (unobserved). Generally, \qnns\ are built by wrapping classical neural-network layers (e.g., fully connected layers) with Variational Quantum Circuits (\vqcs). Typical \vqcs\ consist of a finite sequence of unitary blocks $U(\bm{\theta})=\prod_{l=1}^{L} U_l(\bm{\theta}_l)$, where $\bm{\theta}_l\!\in\!\mathbb{R}^{NR}$ collects the parameters (e.g., $R$ rotation angles per qubit across $N$ qubits) at layer $l$, and the full parameter matrix is $\bm{\theta}=(\bm{\theta}_1,\ldots,\bm{\theta}_L)\!\in\!\mathbb{R}^{LNR}$. Each block can be written as $U_l(\bm{\theta}_l)=\exp(-i\,\bm{\theta}_l^\top \bm{V}_l)$ with a vector of Hermitian generators $\bm{V}_l$; or, in the single-parameter case, $U_l(\theta_l)=e^{-i\theta_l V_l}$ with Hermitian $V_l$.

To optimize the circuit, a standard variational objective $E(\bm{\theta})$ is the expectation of a Hermitian operator $H$ under the variational state $U(\bm{\theta})|0\rangle$:
\begin{equation}
    E(\bm{\theta}) = \langle 0|\, U(\bm{\theta})^{\dagger}\, H\, U(\bm{\theta}) \,|0\rangle.
\label{eqn:loss_fn}
\end{equation}
For $K$-class closed-set classification, we empirically compare the predictive probabilities $\bm{p}(\hat{Y}\!\mid\!X)\!\in\!\mathbb{R}^{|X| \times K}$, which is the measurement statistics of $U(\bm{\theta})$, with the one-hot encoded noisy labels $\mathbf{\tilde{Y}}\!\in\!\{0,1\}^{|X|\times K}$ using cross-entropy loss:
\begin{equation}
    \mathcal{L}_{\mathrm{CE}}(\bm{\theta}) 
    = -\frac{1}{|X|} \sum_{i=1}^{|X|} \sum_{k=1}^{K} \tilde{y}_{i,k}\, \log p_{i,k},
\label{eqn:ce}
\end{equation}
where $p_{i,k} = p(\hat{Y}=k \mid X=x_i;\bm{\theta})$ denotes the predictive probability for the $i$-th instance $x_i$ and class $k$, generated by the model parameterized by $\bm{\theta}$.

This formulation compares the predictive distribution with the observed noisy one-hot labels and is the standard objective for supervised classification with noisy annotations.

\noindent \textbf{Loss Correction with Class-conditional Noise.}
Training \qnns\ with noisy labels inevitably risks fitting the corrupted labels rather than the latent true labels. To mitigate this issue, \citet{patrini2017making} introduced a \emph{forward loss correction} approach that corrects the predictive probabilities $\bm{p}(\hat{Y}\mid X)$ through $\bm{p}(\tilde{Y}\mid X) = \bm{p}(\hat{Y}\mid X)\mathbf{T}$, where $\mathbf{T}\!\in\![0,1]^{K \times K}$ is row-stochastic with entries $\mathbf{T}_{jk} = p(\tilde{Y}=k \mid \hat{Y}=j)$ for all $j,k \in \{1,\ldots,K\}$. Thus, each row of $\mathbf{T}$ describes the distribution of observed noisy labels conditioned on a predicted latent class. Due to the unobservability of latent true labels, such an approach estimates $\mathbf{T}$ using predictive probabilities $\bm{p}(\hat{Y}\mid X)\!\in\!\Delta^{K-1}$ as a proxy for true probabilities $\bm{p}(Y \mid X)$, where probability simplex $\Delta^{K-1} = \{ \bm{p}\!\in\!\mathbb{R}^K \,|\, p_i \ge 0,\; \sum_{i=1}^K p_i = 1 \}$. Therefore, the effectiveness of loss correction depends critically on the accuracy of $\mathbf{T}$. A common strategy constructs $\mathbf{T}$ using \emph{anchors}, defined as:
\begin{definition}[Anchor point]
An instance $x^{(j)}$ is an \emph{anchor point} for class $j$ if $p(\hat{Y}=j \mid X=x^{(j)})=1$. 
When an anchor point $x^{(j)}$ exists for every class $j$, the $j$-th row of the noise transition matrix satisfies $\mathbf{T}_{j\cdot}=p(\tilde{Y}=\cdot \mid \hat{Y}=j)$, where $\mathbf{T}_{j\cdot}\!\in\!\Delta^{K-1}$ is a probability vector with entries $\mathbf{T}_{jk}=p(\tilde{Y}=k \mid \hat{Y}=j)$. 
\label{def:anchor}
\end{definition}
The anchor concept is a common practice in prior loss-correction methods used for estimating $\mathbf{T}$; our method does not require anchors.

\noindent \textbf{Tools from Probability Theory.} Below, we review the definition of martingale (supermartingale), along with key tools relevant to our work. %We adapt the descriptions from \cite{williams1991probability}.

\begin{definition}[Martingale, \cite{williams1991probability}]
  Let $\{M^{(t)}\}_{t\geq 1}$ be a stochastic process w.r.t. a filtration $\{\mathcal{F}^{(t)}\}_{t\geq 1}$ on a probability space $(\Omega, \mathcal{F}, P)$. The process $\{M^{(t)}\}$ is called a \emph{martingale} if (i) $\{M^{(t)}\}$ is adapted, (ii) $\mathbb{E}[|M^{(t)}|] < \infty$, for $\forall\ t \in \mathbb{Z}^+$, (iii) $\mathbb{E}[M^{(t+1)} \mid \mathcal{F}^{(t)}] = M^{(t)}$, almost surely for $\forall\ t \in \mathbb{Z}^+$.
  
  If (iii) is replaced by $\mathbb{E}[M^{(t+1)} \mid \mathcal{F}^{(t)}] \leq M^{(t)}$ almost surely, we say $\{M^{(t)}\}$ is a \emph{supermartingale}.
\label{def:martingale}
\end{definition}

\begin{theorem}[Doob's Forward Convergence Theorem, \cite{williams1991probability}]
  Let $\{M^{(t)}\}_{t\geq 1}$ be an $L^1$-bounded \emph{supermartingale} (in Def.~\ref{def:martingale}). Then, almost surely, the limit $M^{(\infty)} = \lim_{t \to \infty}M^{(t)}$ exists and is finite.
\label{thm:fct}
\end{theorem}

\section{Our Proposed Framework}
{\bf Our key idea} is to leverage high-confidence predictions to adaptively refine the transition matrix with guaranteed convergence, without relying on hard-to-find anchor points. Prior works, such as T-revision~\cite{xia2019anchor}, empirically validate that estimating the transition matrix using instances with high predictive confidence can effectively improve the performance of loss correction. This {\bf intuition} is natural: as predictive entropy decreases, the certainty of predictions increases, making them more reliable for guiding correction in training. Building on this intuition, we propose a new supermartingale-based label transition (SLT) framework that models the entropy reduction process of predictive probabilities as a supermartingale. To formulate the supermartingale process, we monitor the normalized predictive entropy ($s$), across $|X|$ instances and $K$ classes:
\begin{equation}
s[\bm{p}(\hat{Y}|X)] = \frac{-1}{|X| \log K} \sum_{i=1}^{|X|} \sum_{k=1}^K p_{i,k} \log p_{i,k}.
\label{eqn:entropy}
\end{equation}

This signal provides a global measure of model certainty: lower entropy corresponds to more confident predictions. Whenever the current $s$ falls below its historical minimum $S$, we update the transition matrix accordingly. By modeling the trajectory of historical minimum $S$ as a supermartingale, SLT leverages the monotonic accumulation of increasingly confident predictions to refine the transition matrix in an adaptive manner. The {\bf key advantage} of this perspective lies in the favorable properties of supermartingales: \emph{monotonicity and guaranteed convergence}. We further provide a theoretical analysis showing that the noisy transitions induced by our supermartingale formulation converge to a steady state, thereby ensuring stable loss correction.

\begin{figure}[h]
\centering
\begin{minipage}[t]{0.5\textwidth}
    \centering
    \vspace{0pt}
    \includegraphics[width=\linewidth]{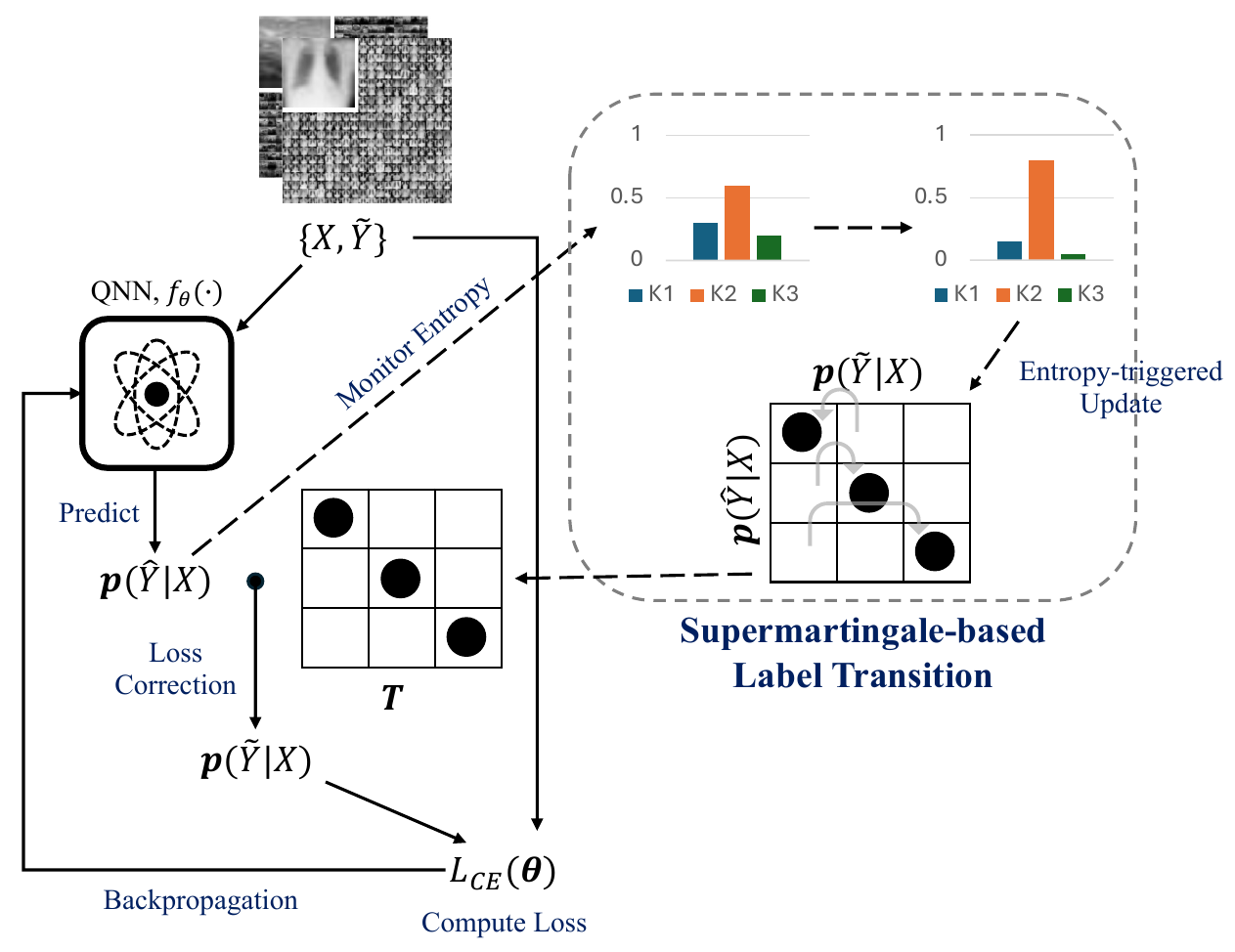}
    \caption{Overall workflow of our proposed framework, SLT. Solid arrows denote the main training process, and dashed arrows indicate the entropy-driven transition-matrix refinement.}
    \label{fig:framework}
\end{minipage}
\hfill
\begin{minipage}[t]{0.46\textwidth}
    \vspace{0pt}
    \begin{algorithm}[H]
    \small
    \caption{Core steps of SLT algorithm.}
    \label{algo:slt}
    \begin{algorithmic}[1]
    \REQUIRE Noisy dataset $\{\mathbf{X}, \mathbf{\tilde{Y}}\}$, \qnn\ $f_{\bm{\theta}}$, the number of training epoch $T$.
    \STATE Initialize transition matrix $\mathbf{T}$ and score $S^{(0)}$;
    \FOR{$t = 1$ to $T$}
        \STATE $\bm{p}(\hat{Y}|X) \leftarrow \mathrm{Softmax}(f_{\bm{\theta}}(\mathbf{X}))$;
        \STATE $\bm{p}(\tilde{Y}|X) \leftarrow \bm{p}(\hat{Y}|X)\mathbf{T}^{(t)}$;
        \STATE Update $\bm{\theta}$ by minimizing $\mathcal{L}_{\mathrm{CE}}$;
        \STATE Compute entropy score $s^{(t)}$ by Eq.~(\ref{eqn:entropy});
        \IF{$s^{(t)} < S^{(t-1)}$}
            \STATE Update $\mathbf{T}^{(t)}$ by Eqs.~(\ref{eqn:count})--(\ref{eqn:ma});
            \STATE $S^{(t)} \leftarrow s^{(t)}$;
        \ENDIF
    \ENDFOR
    \STATE \textbf{return} $\mathbf{T}, f_{\bm{\theta}}$.
    \end{algorithmic}
    \end{algorithm}
\end{minipage}
\vspace{-5pt}
\end{figure}
We present the workflow of our method in Fig.~\ref{fig:framework} and provide the algorithmic details in Algo.~\ref{algo:slt}. 
Given a noisy labeled dataset $\{\mathbf{X}, \mathbf{\tilde{Y}}\}$ and a \qnn\ $f_{\bm{\theta}}(\cdot)$, we first initialize a warm-up transition matrix $\mathbf{T}$ (we provide implementation details in {\bf Appendix~\ref{app:exp}}) ({\bf line~1}). After initialization, we perform $T$ training iterations for loss correction ({\bf line~2}). In each iteration $t$, we multiply (dot product) the predictive probabilities $\bm{p}(\hat{Y}\!\mid\!X)\!\in\!\mathbb{R}^{|X| \times K}$, which is generated by the \qnn\ ({\bf line~3}), with a transition matrix $\mathbf{T}^{(t)}\!\in\!\mathbb{R}^{K \times K}$ that models the noisy transition patterns between noisy labels and latent true labels ({\bf line~4}). After obtaining the corrected predictive probabilities $\bm{p}(\tilde{Y}\!\mid\!X)$, we compute the cross-entropy loss between the corrected predictions and noisy labels by Eq.(~\ref{eqn:ce}), and then update the model parameters via a gradient step with learning rate $\alpha$ ({\bf line~5}). To monitor the entropy reduction process, we calculate the normalized predictive entropy $s^{(t)}$ by Eq.~(\ref{eqn:entropy}) ({\bf line~6}). If $s^{(t)}$ is smaller than its historical minimum $S^{(t-1)}$ from previous iterations ({\bf line~7}), indicating that the predictions have become more confident, we update the transition matrix $\mathbf{T}^{(t)}$ ({\bf line~8}) and the historical minimum entropy $S^{(t)}$ ({\bf line~9}). Finally, we output the refined transition matrix $\mathbf{T}$ together with the corrected \qnn\ ({\bf line~12}).
We further analyze time and space complexity in \textbf{Appendix~\ref{app:proof}}.

\noindent \textbf{Dynamic refinement of the transition matrix.}
To refine the transition matrix $\mathbf{T}$ adaptively, at each step, we construct the dynamic transition matrix $\mathbf{T}'$ based on the current predicted labels (with lowest entropy) and the noisy labels $\tilde{Y}$. For each class pair $(j,k)$, the entry $\mathbf{T}'[j,k]$ is defined as the normalized co-occurrence between predicted labels $\hat{y}=j$ and noisy labels $\tilde{y}=k$:
\begin{equation}
    \mathbf{T}'[j,k] = \frac{\sum_{i=1}^{|X|} \mathbf{1}\{\hat{y}_i=j,\,\tilde{y}_i=k\}}{\sum_{i=1}^{|X|} \mathbf{1}\{\hat{y}_i=j\}},
\label{eqn:count}
\end{equation}
where $\mathbf{1}({\cdot})$ denotes an indicator function. Intuitively, $\mathbf{T}'[j,k]$ estimates the probability that an instance predicted as class $j$ is labeled as noisy class $k$.

To mitigate fluctuations due to frequent refinement, we update the global transition matrix $\mathbf{T}$ in a moving average manner as:
\begin{equation}
  \mathbf{T} \leftarrow (1 - \eta)\mathbf{T} + \eta\mathbf{T}',
\label{eqn:ma}
\end{equation}
where $\eta\!\in\![0,1]$ is a step size controlling the trade-off between stability and adaptability. A smaller $\eta$ emphasizes stability by retaining more of the historical matrix, while a larger $\eta$ accelerates adaptation to the latest predictions.

These update rules enable SLT to progressively refine the transition matrix using increasingly confident predictions, while preserving robustness against stochastic noise.
%Besides, we provide further discussion on the update frequency in the {\bf Appendix}.

\noindent \textbf{Theoretical analysis of our framework.}
We first present the necessary results and further discuss their implications.
%Complete proofs are provided in the {\bf Appendix}.

By Eq.~(\ref{eqn:entropy}), we introduce a random process based on the historical minimum of $s$. Let $s^{(t)}$ be the normalized predictive entropy computed at the $t$-th iteration. We define the historical minimum value up to iteration $t$ as $ S^{(t)} := \min_{1 \leq \tau \leq t}s^{(\tau)} = \min\{ s^{(0)}, s^{(1)}, ..., s^{(t)} \}$.

From this definition, $S^{(t)}$ is a monotonically non-increasing sequence. We can express the change at the $t$-th step as $S^{(t)} = S^{(t-1)} - \Delta^{(t)}$, where $\Delta^{(t)} := S^{(t-1)} - S^{(t)} = \bigl( S^{(t-1)} - s^{(t)} \bigr)_+ \geq 0$ and $(\cdot)_+$ denotes $\max(\cdot, 0)$. Unrolling this recursion yields the explicit form $S^{(t)} = S^{(0)} - \sum_{\tau=1}^{t}\Delta^{(\tau)}$.

Next, we will show that the process $\{S^{(t)}\}_{t \ge 1}$ forms a supermartingale, which is a key step towards proving the convergence of our method.

\begin{lemma}[Supermartingale Property]\label{lem:spmg}
Let $\{S^{(t)}\}_{t \ge 1}$ be a sequence of random variables and the natural filtration $\mathcal{F}^{(t)} = \sigma\bigl(S^{(0)}, \Delta^{(1)}, \cdots, \Delta^{(t)} \bigr)$. Then, the sequence $\{S^{(t)}\}_{t \ge 1}$ is a supermartingale with respect to the filtration $\{\mathcal{F}^{(t)}\}_{t \ge 1}$.
\end{lemma}

\begin{theorem}[Convergence of the Transition]\label{thm:convergence}
The sequence of transition matrices 
$\{ {\bf T}^{(t)} \}_{t\geq1}$ in Algo.~\ref{algo:slt} converges almost surely to a fixed-point matrix ${\bf T}^{*}$.
\end{theorem}

Thm.~\ref{thm:convergence} establishes convergence of the update dynamics, thereby theoretically ensuring the stability of our proposed method. We provide full proof in \textbf{Appendix~\ref{app:proof}}.
Notably, our framework aims to improve classification and training stability of \qnns\ under noisy labels, rather than forcing predictions to approach the unobserved clean labels. Since overly aggressive fitting to latent clean labels can increase overfitting risk, this work prioritizes stability-aware training under label noise.

\section{Experiment}
\label{sec:exp}
In this section, we first introduce the experimental settings and further present our results in detail.

\begin{table*}[t]
\centering
\newcommand{\std}[1]{{\scriptsize\ (#1)}}
\caption{Comparison with baselines (CE, Forward) and competing methods on Breast, Pneumonia, and Retina, under three noise settings (UN: Uniform, CF: Cyclic-flipping, CM: Custom-mapping) at noise levels 10\%/30\%/50\%. We report mean(std) of F1~(\%) over five runs.}
\label{tab:compare1}
%\footnotesize
\scriptsize
\setlength{\tabcolsep}{1.2pt}
\begin{tabular}{c|c|ccc|ccc|ccc}
\toprule
\multirow{2}{*}{\textbf{Noise}} & \multirow{2}{*}{\textbf{Models}}
& \multicolumn{3}{c|}{\textbf{Breast}}
& \multicolumn{3}{c|}{\textbf{Pneumonia}}
& \multicolumn{3}{c}{\textbf{Retina}} \\
\cmidrule(lr){3-5}\cmidrule(lr){6-8}\cmidrule(l){9-11}
 & & 10\% & 30\% & 50\% & 10\% & 30\% & 50\% & 10\% & 30\% & 50\% \\
\midrule
\multirow{9}{*}{\textbf{UN}}
 & CE & 66.45\std{3.54} & 61.21\std{5.07} & 47.19\std{2.71} & 84.87\std{2.56} & 83.69\std{4.26} & 60.82\std{2.85} & 27.39\std{0.73} & 25.14\std{2.47} & 22.71\std{3.25}\\
 & Forward & 67.07\std{2.17} & 64.33\std{6.12} & 48.00\std{2.56} & 85.18\std{1.88} & 84.12\std{1.89} & 81.32\std{0.67} & 28.77\std{1.78} & 26.95\std{3.18} & 24.51\std{2.68}\\
 & T-Revision & 64.82\std{3.81} & 60.06\std{4.95} & 43.86\std{3.68} & 74.24\std{2.93} & 68.46\std{5.20} & 66.22\std{6.34} & 21.71\std{3.09} & 20.60\std{2.17} & 20.17\std{3.84}\\
 & Dual-T & 67.06\std{1.23} & 64.30\std{5.74} & 49.25\std{2.53} & 85.64\std{1.64} & 85.27\std{2.85} & 81.71\std{1.41} & 28.77\std{1.78} & 26.95\std{3.18} & 24.51\std{2.68}\\
 & VolMinNet & 66.71\std{1.52} & 63.95\std{3.41} & 47.43\std{2.15} & 85.02\std{1.24} & 84.60\std{1.82} & 76.00\std{2.51} & 28.54\std{1.80} & 25.80\std{2.17} & 23.54\std{3.59}\\
 & TVR & 68.06\std{1.56} & 61.74\std{3.50} & 50.61\std{3.29} & 76.65\std{2.38} & 76.93\std{2.32} & 68.62\std{3.80} & 20.90\std{6.73} & 18.09\std{5.81} & 15.91\std{4.18}\\
 & BLTM & 67.31\std{4.36} & 59.42\std{5.29} & 45.02\std{2.76} & 83.59\std{2.32} & 83.21\std{3.87} & 65.92\std{8.62} & 24.27\std{1.98} & 21.61\std{4.69} & 17.25\std{8.35}\\
 & CCR & 67.85\std{2.19} & 62.26\std{3.89} & 51.48\std{3.48} & 84.86\std{2.71} & 85.16\std{1.57} & 79.54\std{2.54} & 27.68\std{2.86} & 27.03\std{2.02} & 24.20\std{3.50}\\
 & Ours & \textbf{69.14\std{2.71}} & \textbf{66.04\std{4.34}} & \textbf{52.24\std{4.41}} & \textbf{86.31\std{1.63}} & \textbf{86.27\std{1.81}} & \textbf{83.90\std{2.09}} & \textbf{33.71\std{2.93}} & \textbf{28.29\std{2.23}} & \textbf{25.57\std{2.88}}\\
\midrule
\multirow{9}{*}{\textbf{CF}}
 & CE & 65.64\std{4.29} & 59.14\std{2.91} & 46.70\std{3.66} & 82.75\std{0.77} & 71.42\std{2.31} & 45.31\std{3.87} & 27.40\std{3.41} & 26.54\std{2.55} & 23.38\std{2.63}\\
 & Forward & 66.29\std{2.01} & 60.82\std{5.02} & 49.87\std{2.88} & 83.13\std{2.24} & 79.46\std{2.72} & 43.23\std{5.28} & 29.66\std{1.54} & 29.46\std{1.56} & 21.03\std{2.98}\\
 & T-Revision & 52.18\std{1.10} & 51.58\std{2.45} & 43.73\std{4.43} & 78.46\std{1.02} & 69.46\std{2.53} & 36.22\std{3.54} & 24.48\std{0.75} & 22.46\std{1.25} & 21.46\std{1.64}\\
 & Dual-T & 66.17\std{3.55} & 60.51\std{3.16} & 48.17\std{4.61} & 83.90\std{2.43} & 79.34\std{1.89} & 35.78\std{6.82} & 25.45\std{1.66} & 24.93\std{5.33} & 18.13\std{7.53}\\
 & VolMinNet & 65.89\std{2.31} & 60.68\std{2.59} & 46.27\std{5.31} & 84.37\std{1.47} & 79.50\std{2.74} & 34.44\std{7.50} & 25.62\std{1.96} & 23.81\std{5.69} & 17.49\std{7.77}\\
 & TVR & 64.18\std{2.18} & 59.76\std{6.17} & 43.21\std{4.97} & 75.28\std{2.65} & 55.29\std{2.67} & 44.30\std{5.74} & 18.47\std{4.46} & 19.96\std{6.60} & 19.03\std{7.43}\\
 & BLTM & 67.91\std{3.79} & 58.45\std{5.29} & 49.63\std{2.51} & 81.26\std{1.59} & 73.75\std{4.23} & 46.62\std{2.59} & 25.41\std{4.15} & 25.31\std{5.33} & 17.18\std{7.04}\\
 & CCR & 66.85\std{1.36} & 60.63\std{3.42} & 47.26\std{4.46} & 82.97\std{2.86} & 80.16\std{3.12} & 37.93\std{6.89} & 28.12\std{2.11} & 27.27\std{3.64} & 20.92\std{3.55}\\
 & Ours & \textbf{71.22\std{2.89}} & \textbf{63.53\std{3.82}} & \textbf{51.18\std{3.29}} & \textbf{85.53\std{1.30}} & \textbf{81.21\std{2.79}} & \textbf{47.44\std{5.78}} & \textbf{31.84\std{3.66}} & \textbf{30.96\std{2.67}} & \textbf{28.40\std{3.63}}\\
\midrule
\multirow{9}{*}{\textbf{CM}}
 & CE & 65.64\std{4.29} & 59.14\std{2.91} & 46.70\std{3.66} & 82.75\std{0.77} & 71.42\std{2.31} & 45.31\std{3.87} & 27.77\std{1.08} & 26.72\std{1.89} & 22.83\std{1.61}\\
 & Forward & 66.29\std{2.01} & 60.82\std{5.02} & 49.87\std{2.88} & 83.13\std{2.24} & 79.46\std{2.72} & 43.23\std{5.28} & 29.44\std{2.32} & 27.39\std{2.53} & 21.17\std{2.94}\\
 & T-Revision & 52.18\std{1.10} & 51.58\std{2.45} & 43.73\std{4.43} & 78.46\std{1.02} & 69.46\std{2.53} & 36.22\std{3.54} & 23.97\std{4.58} & 19.25\std{4.73} & 18.62\std{5.48}\\
 & Dual-T & 66.17\std{3.55} & 60.51\std{3.16} & 48.17\std{4.61} & 83.90\std{2.43} & 79.34\std{1.89} & 35.78\std{6.82} & 26.18\std{2.61} & 27.39\std{2.53} & 21.17\std{2.94}\\
 & VolMinNet & 65.89\std{2.31} & 60.68\std{2.59} & 46.27\std{5.31} & 84.37\std{1.47} & 79.50\std{2.74} & 34.44\std{7.50} & 26.29\std{3.62} & 25.57\std{0.97} & 20.33\std{3.65}\\
 & TVR & 64.18\std{2.18} & 59.76\std{6.17} & 43.21\std{4.97} & 75.28\std{2.65} & 55.29\std{2.67} & 44.30\std{5.74} & 21.24\std{4.60} & 20.04\std{7.91} & 21.34\std{5.63}\\
 & BLTM & 67.91\std{3.79} & 58.45\std{5.29} & 49.63\std{2.51} & 81.26\std{1.59} & 73.75\std{4.23} & 46.62\std{2.59} & 23.78\std{5.57} & 20.97\std{5.24} & 18.04\std{4.85}\\
 & CCR & 66.85\std{1.36} & 60.63\std{3.42} & 47.26\std{4.46} & 82.97\std{2.86} & 80.16\std{3.12} & 37.93\std{6.89} & 27.90\std{1.23} & 24.58\std{5.99} & 19.65\std{3.90}\\
 & Ours & \textbf{71.22\std{2.89}} & \textbf{63.53\std{3.82}} & \textbf{51.18\std{3.29}} & \textbf{85.53\std{1.30}} & \textbf{81.21\std{2.79}} & \textbf{47.44\std{5.78}} & \textbf{30.56\std{2.78}} & \textbf{29.94\std{1.77}} & \textbf{26.44\std{3.05}}\\
\bottomrule
\end{tabular}
\end{table*}

\noindent \textbf{Dataset.}
We first validate our proposed framework on five MedMNIST datasets: {\bf Breast}MNIST, {\bf Pneumonia}MNIST, {\bf Retina}MNIST, {\bf Derma}MNIST, and {\bf Blood}MNIST~\cite{medmnistv1}. We select these datasets as they reflect the bottleneck of medical images: relatively small scale due to expensive data acquisition and annotation. This setting is more challenging than large-scale natural image datasets, since limited data magnifies the adverse effect of noisy labels, making robust learning particularly difficult.
Data statistics, noise settings, and additional validations under instance-dependent noise (IDN) are presented in {\bf Appendix~\ref{app:exp}}.
Furthermore, we validate our method on CheXpert~\cite{irvin2019chexpert}, a real-world dataset with noisy labels.

\noindent \textbf{Implementation details.}
We employ a backbone \qnn\ consisting of 8 qubits, 2 layers, and 3 rotation gates. The training process involves a 100-epoch warmup phase to build the transition matrix, followed by 100 epochs of training for loss correction. For optimization, we use the Adam optimizer with a learning rate of 0.01 and a batch size of 128. All results are reported as the mean and standard deviation over five independent runs using five distinct random seeds (from 42 to 46) to ensure reproducibility. The optimal hyperparameters for our framework are reported in Tab.~\ref{tab:hyperparams}. We report classification performance using the macro-averaged F1 score, as class imbalance is prevalent in medical image datasets. A higher F1 score implies a better classification.

\noindent \textbf{Baseline and competing methods.}
For fair comparison, we compare our framework with the standard baseline and competing methods that are widely used as anchor-free label transition approaches.
{\bf CE} denotes the approach that we directly train the model with cross-entropy loss without loss correction.
{\bf Forward} loss correction~\cite{patrini2017making} attempts to correct predictive probabilities with an estimated transition matrix.
{\bf T-Revision}~\cite{xia2019anchor} builds the initial transition matrix from high-confident predicted instances and then iteratively refines it with a learnable slack term $\Delta T$ while jointly optimizing the classifier by minimizing the weighted loss.
{\bf Dual-T}~\cite{yao2020dual} jointly leverages forward- and backward-corrected losses to more accurately estimate the noise transition matrix.
{\bf VolMinNet}~\cite{li2021provably} proposes an anchor-free framework that regularizes the cross-entropy loss with the volume of the simplex formed by the columns of the transition matrix.
{\bf TVR}~\cite{zhang2021learning} applies total variation regularization to identify ``cleanest" predictive probabilities without explicitly detecting anchor points.
{\bf BLTM}~\cite{yang2022estimating} leverages Bayes-optimal labels inferred from distilled clean samples to better estimate the transition matrix.
{\bf CCR}~\cite{cheng2022class} estimates the transition matrix using a forward-backward cycle-consistency regularization without relying on anchor points.
We provide hyperparameter settings in {\bf Appendix~\ref{app:exp}}.

\noindent \textbf{Our framework performs robustly against baselines and competing methods under three types of label noise across diverse datasets.}
We evaluate our framework (Ours) against two baselines (Cross-Entropy, i.e., CE, and Forward) and six representative anchor-free label-transition-based methods (T-Revision, Dual-T, VolMinNet, TVR, BLTM, CCR) on five MedMNIST datasets (Breast, Pneumonia, Retina, Derma, Blood) under three types of noise with varying noise ratios (10\%, 30\%, 50\%). All experiments are repeated five times, and we report the mean $\pm$ std of the F1 score. As shown in Tab.~\ref{tab:compare1}, our method outperforms competing approaches across most datasets and noise levels. For example, on Breast and Pneumonia, our framework maintains a clear advantage even under severe noise (50\%), achieving 52.24\% and 83.90\%, respectively, whereas conventional baselines such as CE and Forward suffer from substantial performance degradation. These results demonstrate that our framework achieves superior stability and robustness across diverse datasets and noise conditions, effectively mitigating the adverse effects of label noise on \qnns' performance. Due to the page limit, we present the results of Derma and Blood in \textbf{Appendix~\ref{app:exp}}.

\begin{figure*}[h]
    \centering
    \includegraphics[width=1.0\linewidth]{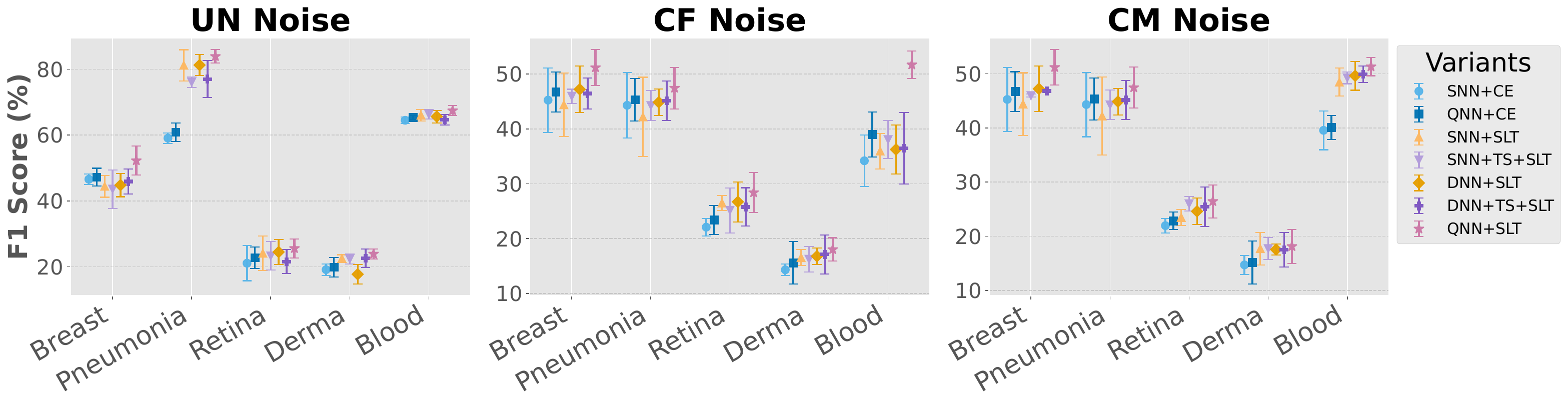}
    \caption{Ablation study on \vqcs\ evaluated by F1 score (mean $\pm$ std) across five datasets under three types of noises, uniform (UN), cyclic-flipping (CF), and custom-mapping (CM) noise, with 50\% noise ratio. We compare our backbone \qnn\ with classical shallow and deep neural networks (NNs) on two training schemes, direct training using cross-entropy (CE) and our framework (SLT), respectively. To assess how smoothness affects NNs' output, we apply temperature scaling (TS) to both NNs before employing SLT.}
    \label{fig:ablation}
\end{figure*}
\noindent \textbf{Ablation studies on backbone models.}
We conduct additional ablation studies to validate the effectiveness of the backbone \qnn\ under three types of noise with a 50\% noise ratio across five datasets. First, we compare our framework, SLT, with a baseline, training with cross-entropy (CE) only, on both classical and quantum models. Second, we compare three backbone models, a classical shallow neural network (SNN), a classical deep neural network (DNN), and a \qnn, using SLT, where the SNN consists of two fully connected layers, while the DNN is a ResNet~\cite{he2016deep} with ten residual blocks.
As shown in Fig.~\ref{fig:ablation}, the results demonstrate the effectiveness of \qnn, confirming its advantage over classical neural networks on datasets with limited instances and noisy annotations.
Besides, we observe that using classical models as the backbone of SLT may lead to performance degradation. We argue that classical models are prone to severe early-stage overconfidence~\cite{wei2022mitigating}, often pushing Softmax outputs to extremes and causing predictive entropy to collapse to nearly zero even under noisy labels. This problem is further amplified in our framework, since SLT relies on entropy reduction to update the transition matrix.
To further examine this overconfidence issue, we apply temperature scaling (TS)~\cite{guo2017calibration} to both SNN and DNN. We fine-tune the temperature on the validation set and select 1.5 in this setting. The results show that although DNN is generally slightly better than SNN in most cases, applying TS does not consistently improve the performance of either backbone in SLT. In contrast, \qnns\ exhibit an inherent ``natural smoothness'' rooted in the Born rule~\cite{born1926quantenmechanik}, where outputs are derived from projection probabilities of quantum states governed by unitary transformations and limited Hilbert space expressivity; this prevents the probability distribution from instantaneously collapsing to a one-hot state.
This gradual growth of confidence matches SLT's monotonic yet steady entropy reduction process well, allowing our framework to capture the true learning trajectory from uncertainty to certainty rather than being misled by the deceptive overfitting signals often seen in classical models. This observation further suggests that SLT has particular advantages when combined with \qnns.
Furthermore, we analyze the trainable parameters for backbone models and discuss how the smoothness affects \qnns\ in \textbf{Appendix~\ref{app:exp}}.

\begin{figure*}[h]
  \begin{subfigure}{1.0\linewidth}
    \centering
    \includegraphics[width=0.19\textwidth]{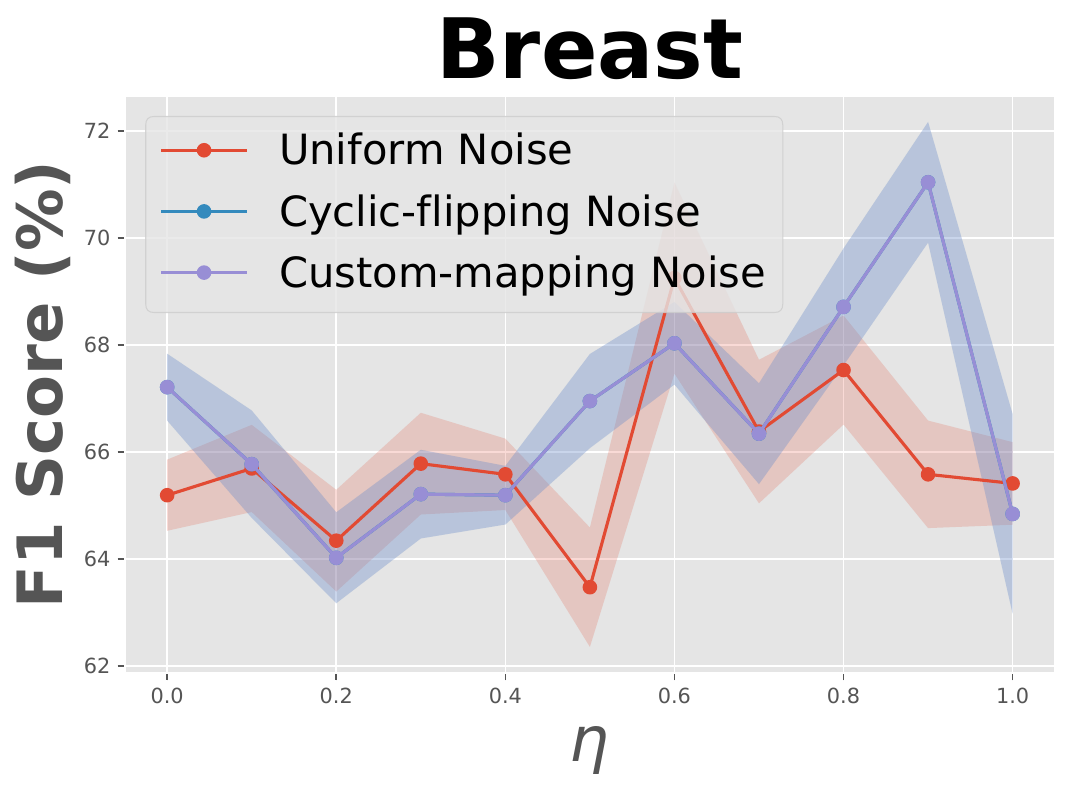}
    \includegraphics[width=0.19\textwidth]{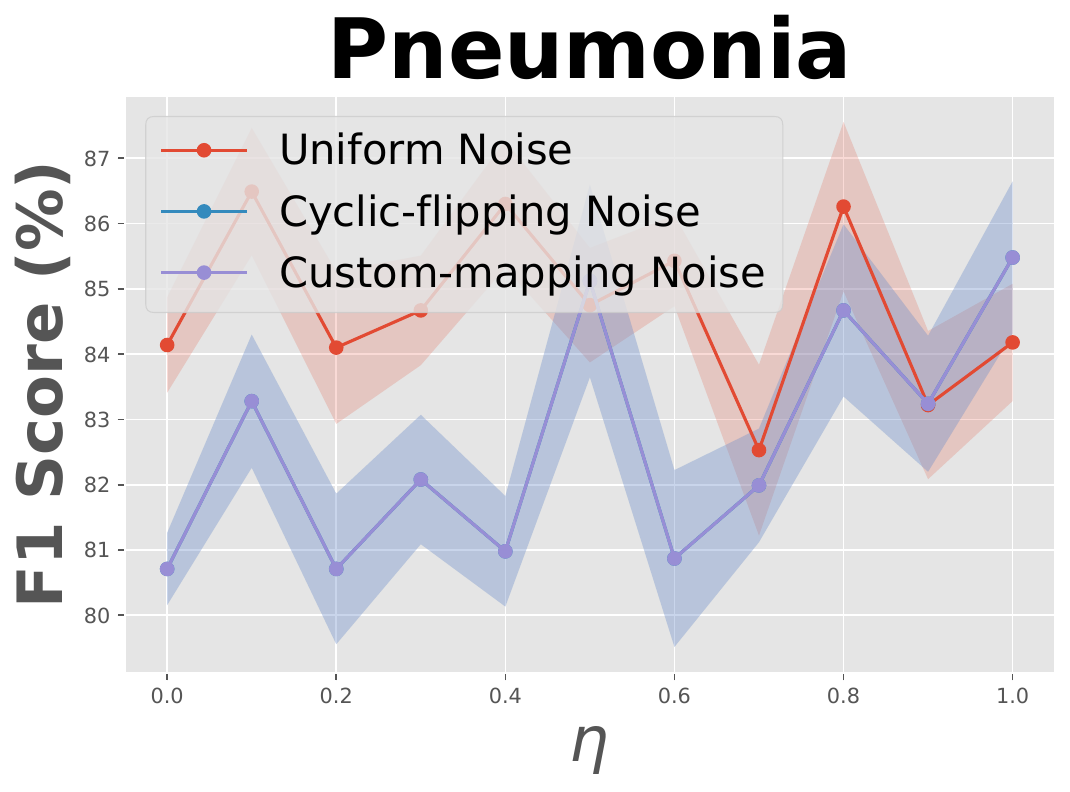}
    \includegraphics[width=0.19\textwidth]{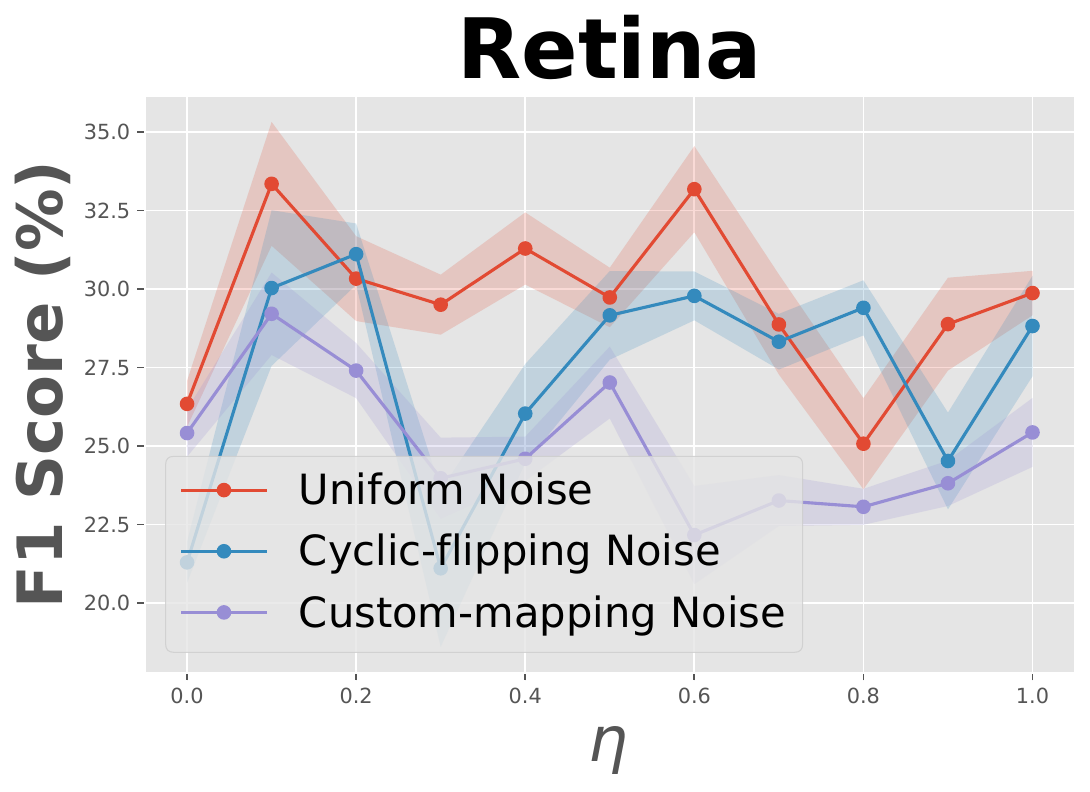}
    \includegraphics[width=0.19\textwidth]{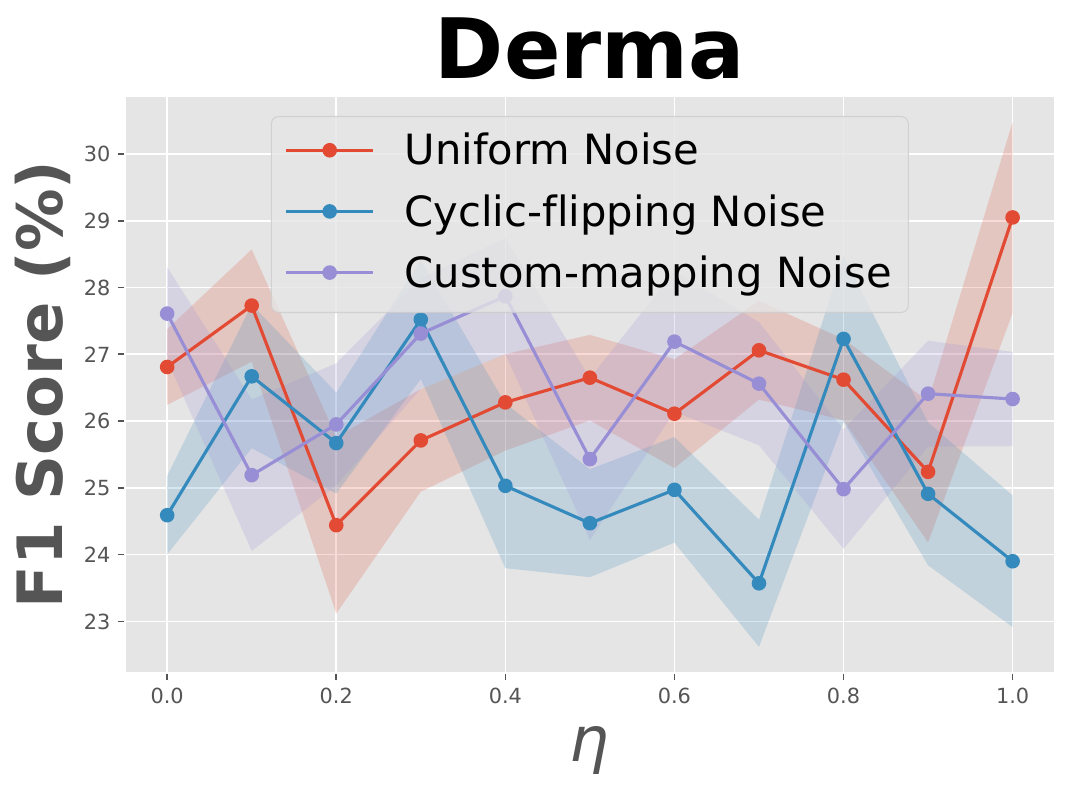}
    \includegraphics[width=0.19\textwidth]{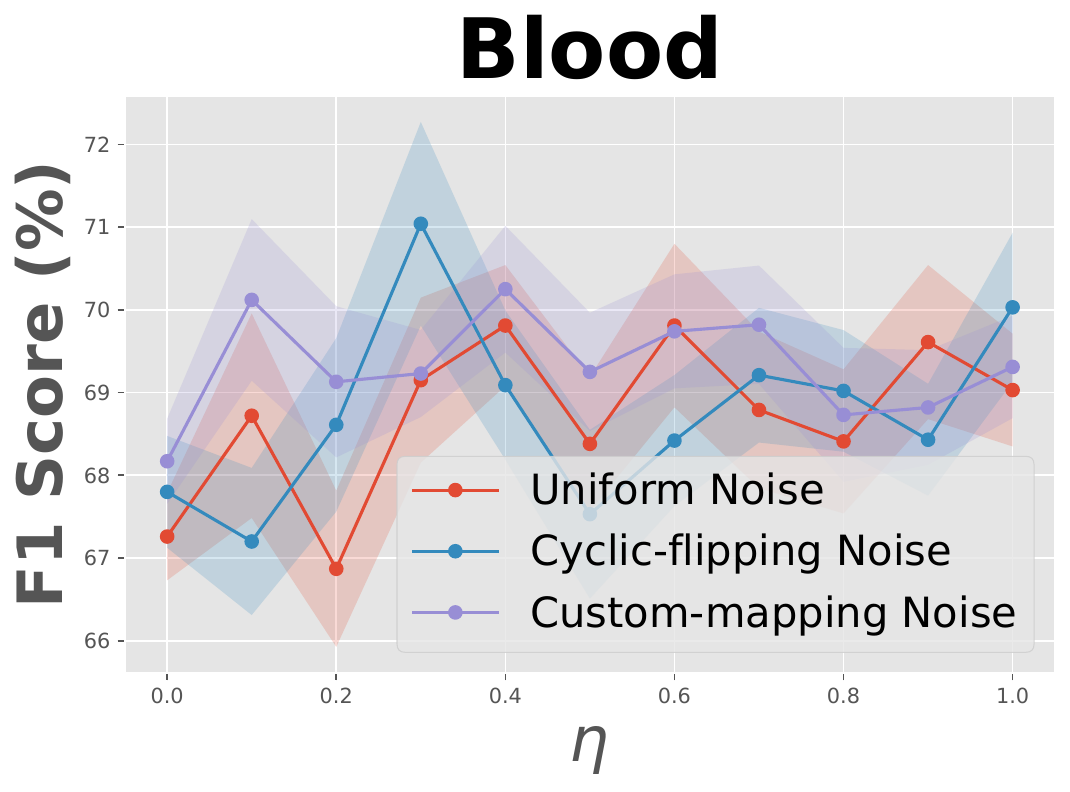}
  \end{subfigure}
\caption{Sensitivity analysis of hyperparameters $\eta$ under three noise types for five datasets. We run the experiments five times and present the mean F1 score with standard deviation in shaded regions.}
\label{fig:eta}
\end{figure*}
\noindent \textbf{Analysis of hyperparameters $\eta$.}
We analyze the sensitivity of hyperparameters $\eta$ under different noise types (Uniform, Cyclic-flipping, and Custom-mapping) with a 10\% noise ratio for five datasets on the validation set. The transition matrix is updated with the ``NDU+PTU'' mode (0.5 del, 10 pat). The update strategies are discussed in \textbf{Appendix~\ref{app:exp}}. Each curve represents the average F1 score, with the shaded area denoting the standard deviation across five runs.
We observe the patterns of $\eta$ in Fig.~\ref{fig:eta}. Overall, the curves show that the optimal $\eta$ varies across datasets. For asymmetric noises such as Cyclic-flipping and Custom-mapping, the performance is identical on Breast and Pneumonia due to binary classification tasks, where pair-wise label flips produce the same effect. Pneumonia tends to favor large values of $\eta$ under pair-wise noise, while more challenging datasets (Retina and Derma) achieve better results at either very small or extreme settings. Notably, although determining the optimal $\eta$ requires hyperparameter tuning, the performance degradation under sub-optimal settings is generally mild, thereby validating the robustness of our method.

\begin{figure*}[h]
  \begin{subfigure}{1.0\linewidth}
    \centering
    \includegraphics[width=0.19\textwidth]{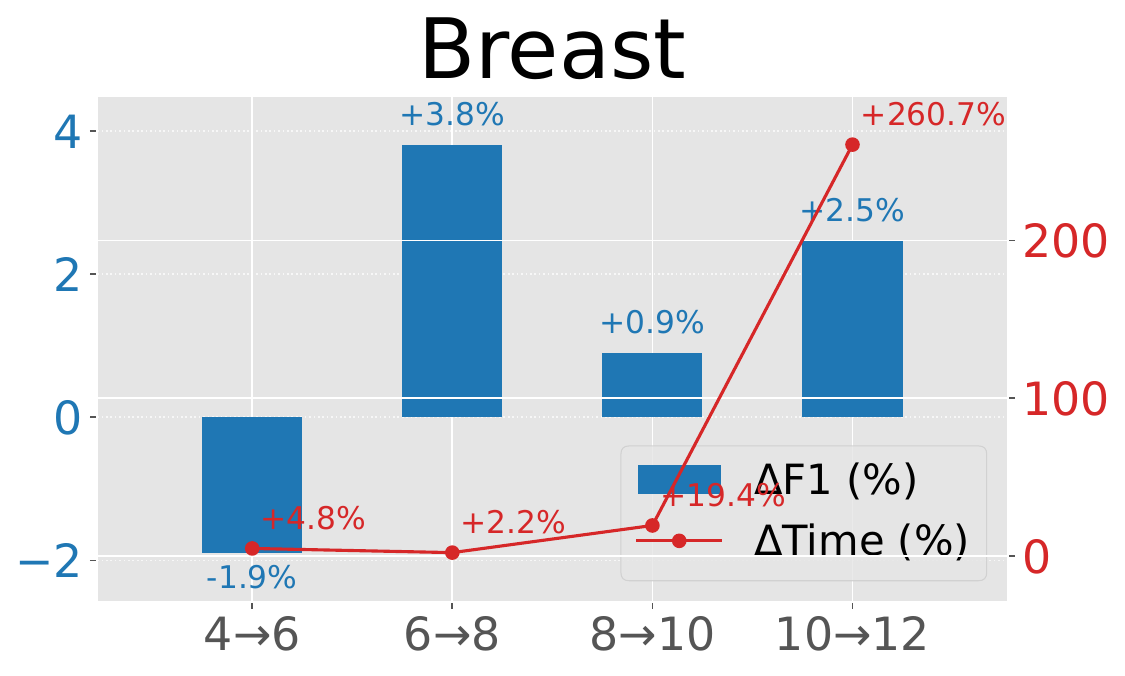}
    \includegraphics[width=0.19\textwidth]{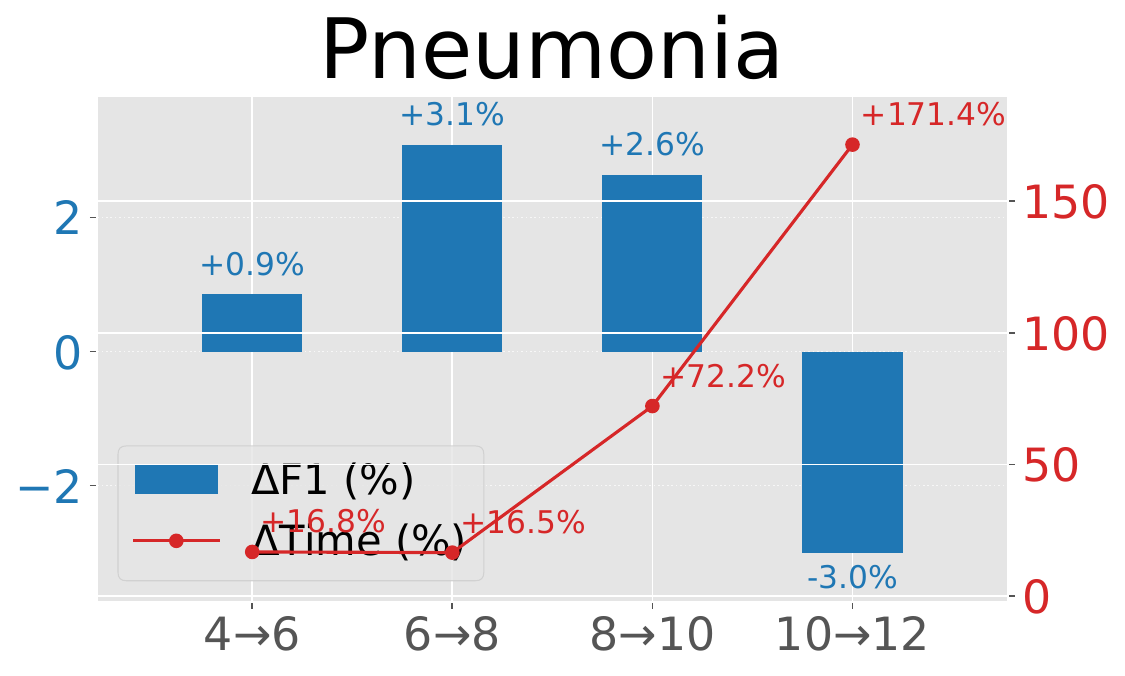}
    \includegraphics[width=0.19\textwidth]{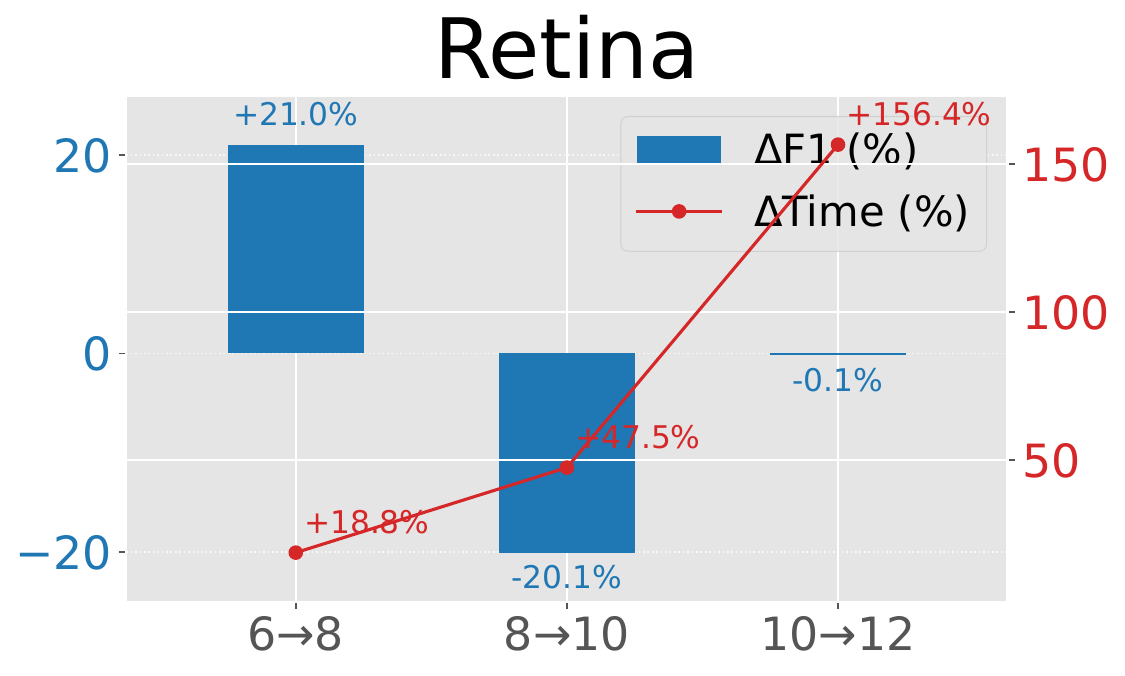}
    \includegraphics[width=0.19\textwidth]{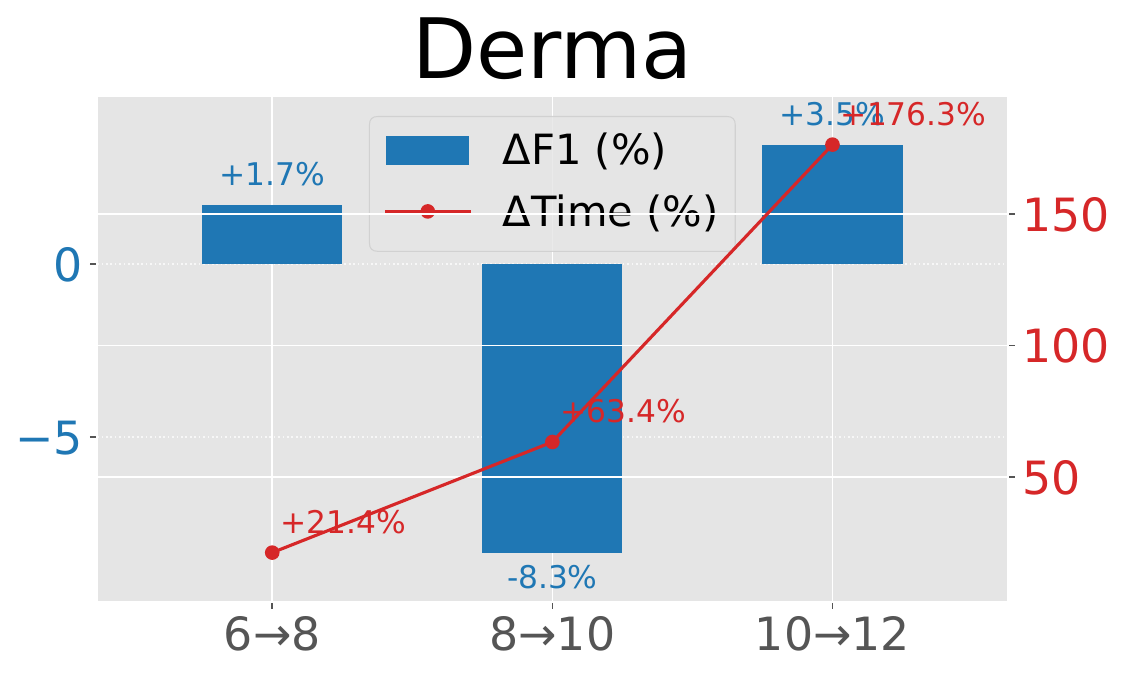}
    \includegraphics[width=0.19\textwidth]{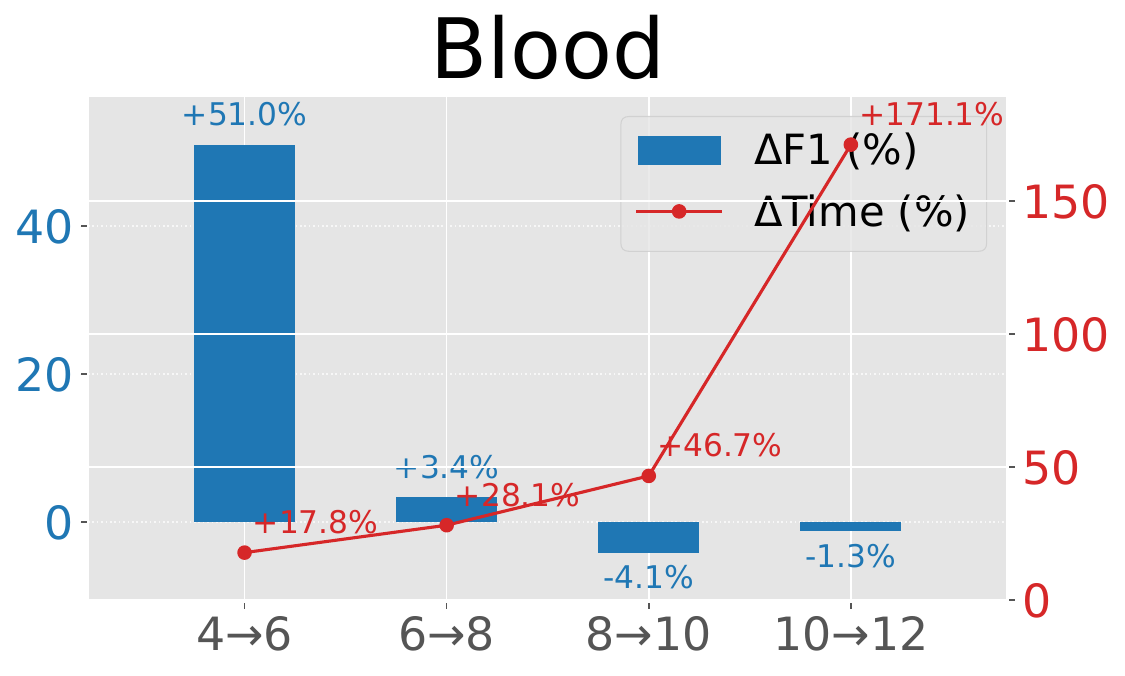}
  \end{subfigure}
  \begin{subfigure}{1.0\linewidth}
    \centering
    \includegraphics[width=0.19\textwidth]{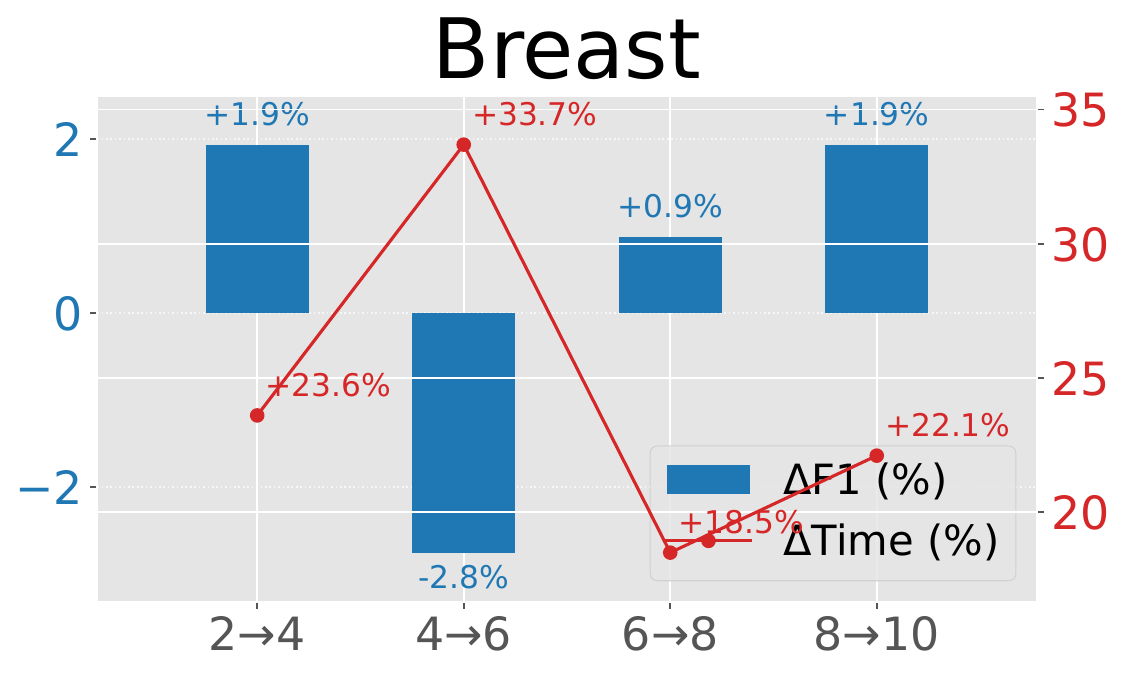}
    \includegraphics[width=0.19\textwidth]{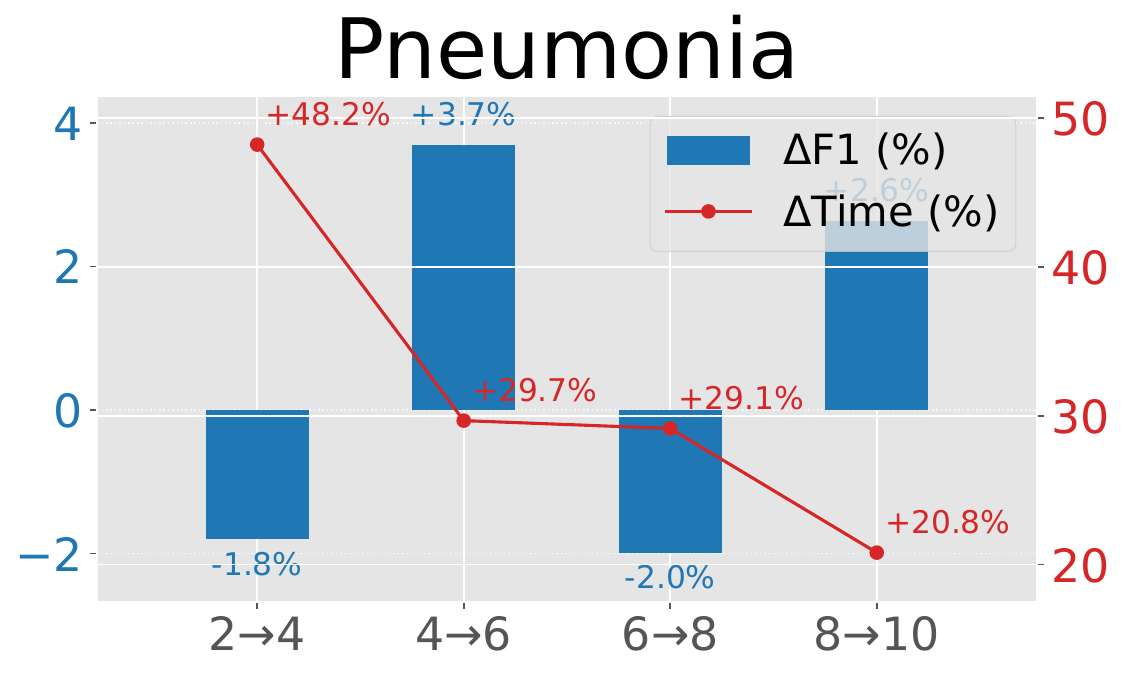}
    \includegraphics[width=0.19\textwidth]{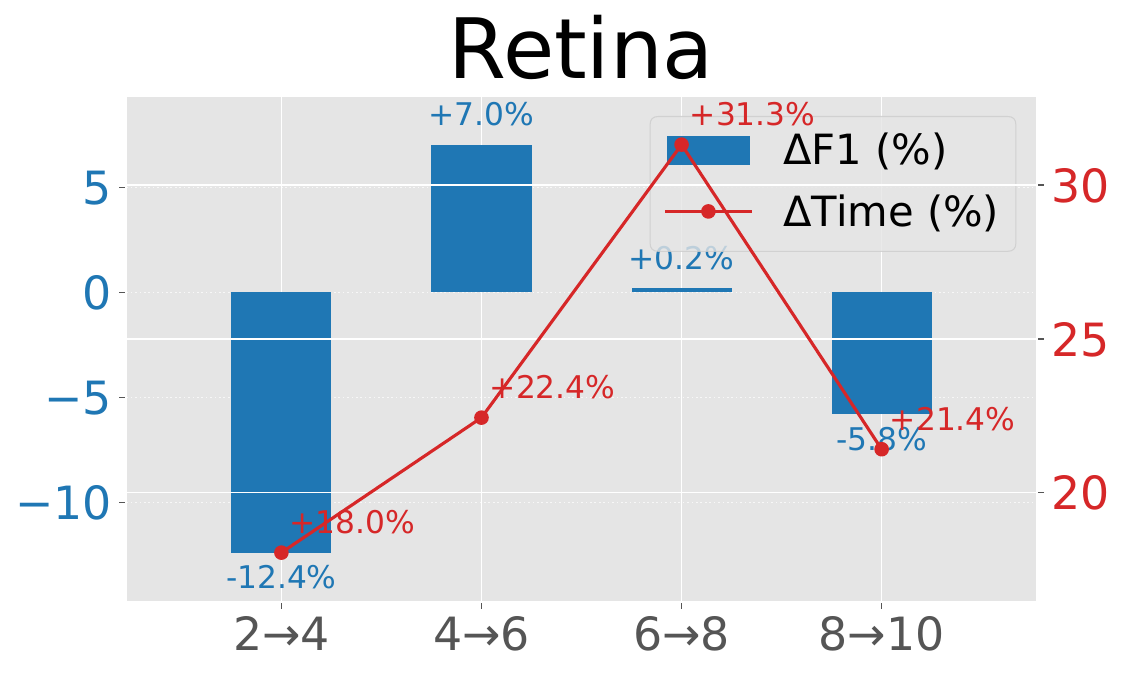}
    \includegraphics[width=0.19\textwidth]{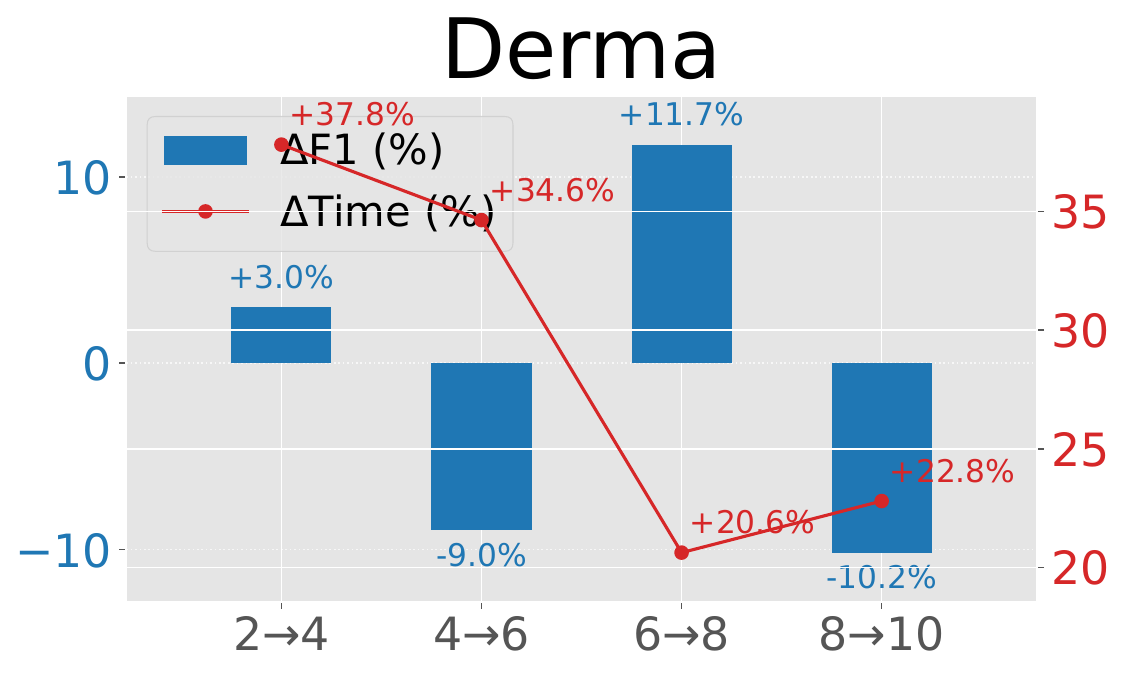}
    \includegraphics[width=0.19\textwidth]{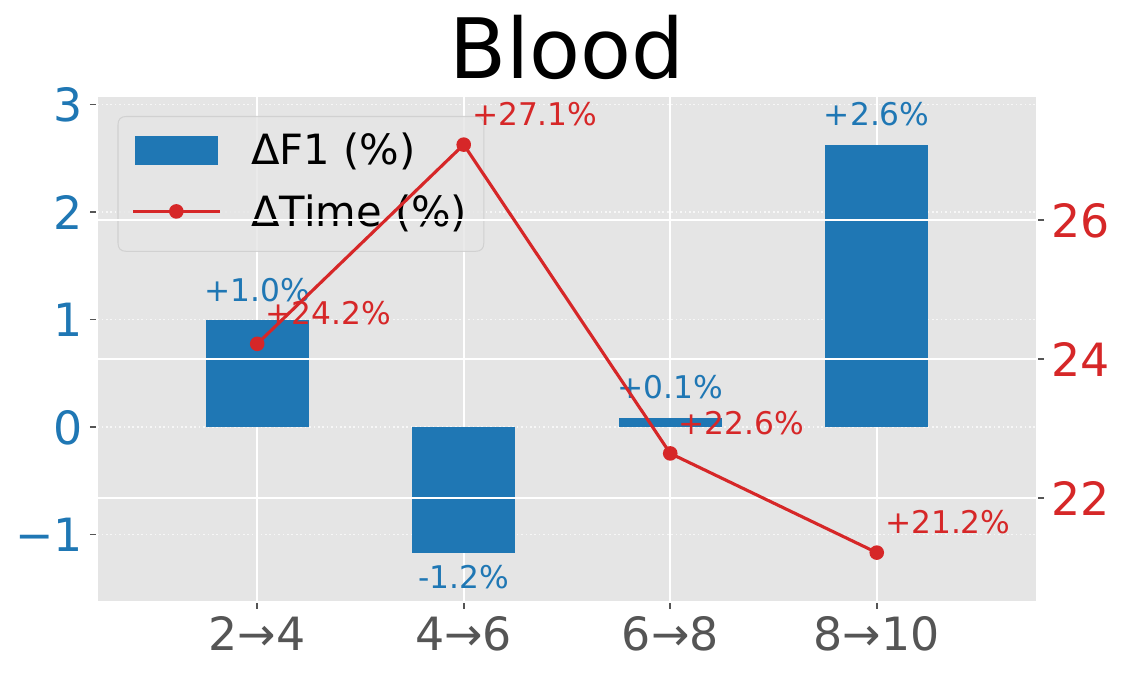}
  \end{subfigure}
\caption{Trade-off analysis under different numbers of qubits (upper row) and layers (lower row). We report incremental F1 gain (\%) in blue bars and Time cost (\%) in red lines as the number of consecutive qubits/layers increases.}
\label{fig:cpr}
\end{figure*}
\noindent \textbf{Trade-off Analysis under different numbers of qubits and layers.}
We first analyze the performance–cost trade-off of \qnns\ on the validation set and report incremental changes in F1 (blue bars) and runtime (red line) as the number of qubits increases from 4 to 12 while fixing the depth to two layers. To isolate the effect of qubit count, we perform a controlled ablation where only the number of qubits varies and all other settings are held constant across runs. The \qnn\ is evaluated under symmetric noise at a 10\% rate and trained with the ``NDU+PTU'' strategy (0.5 del, 10 pat).
In the upper row of Fig.~\ref{fig:cpr}, we observe that moving from 6 to 8 qubits consistently offers the best performance–cost trade-off, whereas further increases show clear diminishing (or even negative) returns with sharp cost inflation. For example, on Breast, employing 8 qubits yields 3.81\% F1 gain for only 2.18\% runtime increase. In contrast, increasing the qubits from 8 to 10 just obtains 0.90\% F1 gain at an additional 19.37\% runtime cost.
Notably, for Retina and Derma, a valid warm-up transition matrix fails to be built, i.e., predicted classes mismatched the classes in noisy labels, so that these cases are excluded in comparison.
Taken together across datasets, 8 qubits emerges as the global ``sweet spot'': before 8, F1 improves with moderate cost; beyond 8, F1 plateaus or declines while runtime escalates steeply. We therefore choose 8 qubits for this study.
Besides, we fix the number of qubits as 8 and vary the number of layers from 2 to 10. As the number of layers increases, the F1 score increases moderately, but the time cost increases hugely. Considering that we aim to enhance the training under noise-label setting with a cost-effective budget, we choose two layers so we can efficiently validate our methods.

\begin{figure*}[h]
  \begin{subfigure}{1.0\linewidth}
    \centering
    \includegraphics[width=0.19\textwidth]{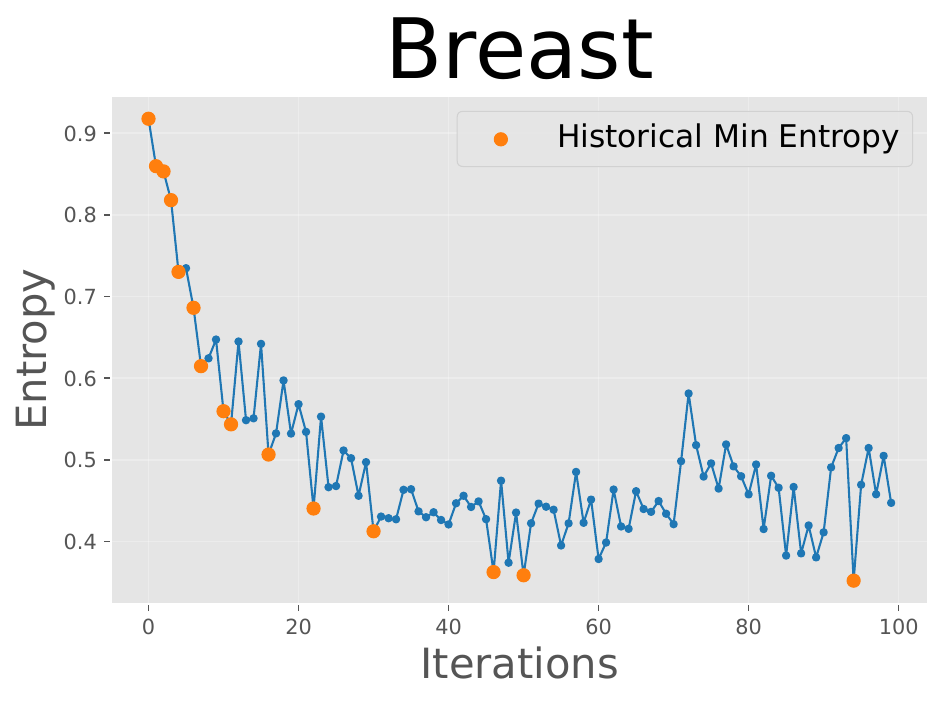}
    \includegraphics[width=0.19\textwidth]{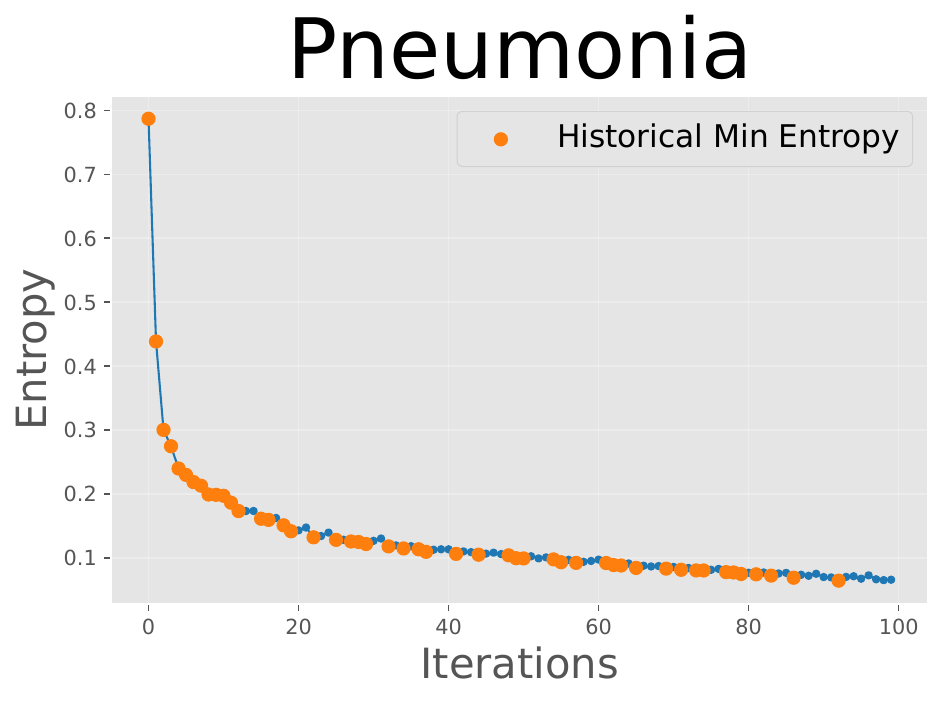}
    \includegraphics[width=0.19\textwidth]{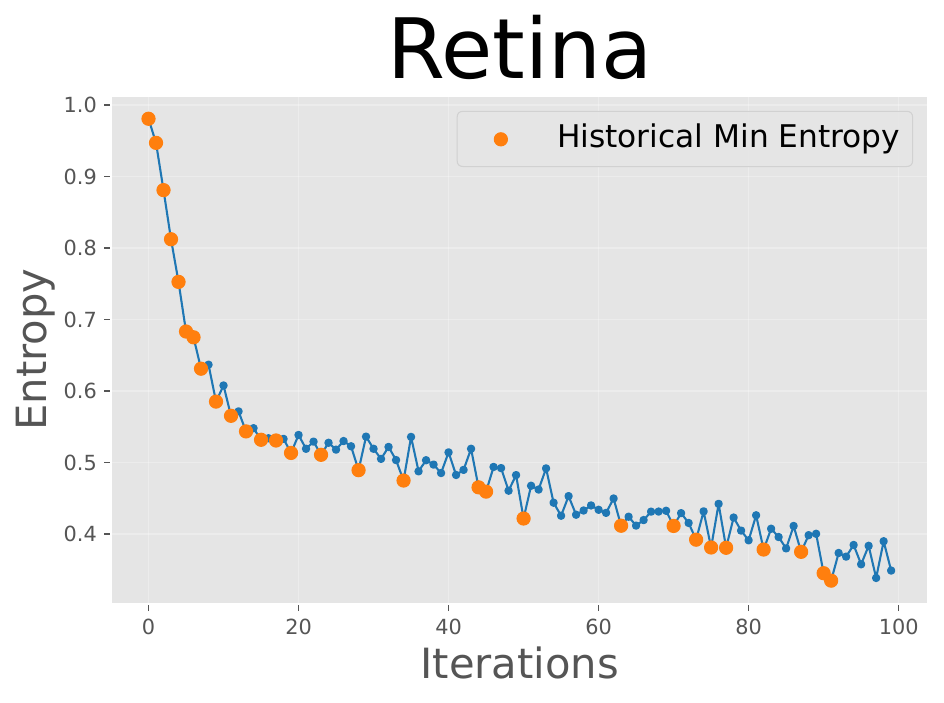}
    \includegraphics[width=0.19\textwidth]{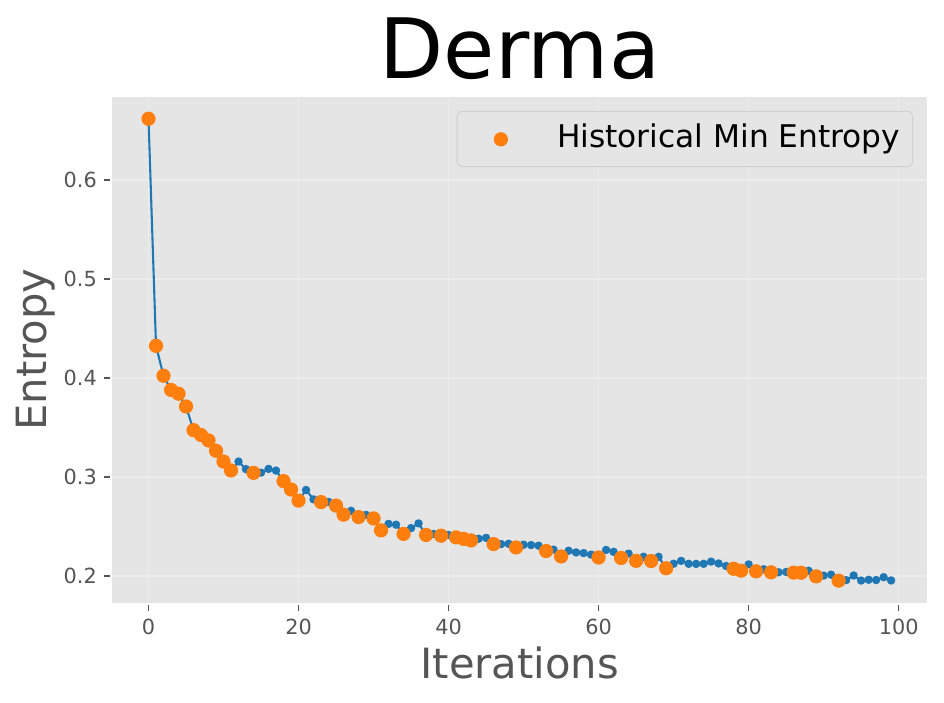}
    \includegraphics[width=0.19\textwidth]{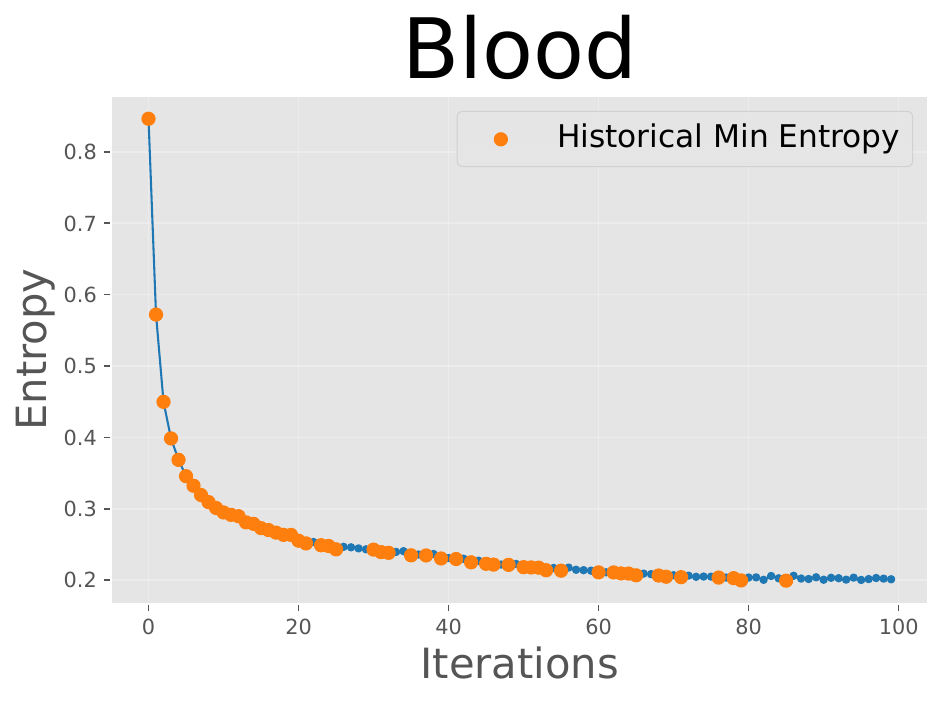}
  \end{subfigure}
\caption{Entropy reduction process. We mark the running historical lowest entropy in orange.}
\label{fig:ent}
\end{figure*}
\noindent \textbf{Discussion on the entropy reduction process.} We present the entropy reduction process across five datasets in Fig.~\ref{fig:ent}. Historical minima up to the past iterations are highlighted in orange. The results validate that the entropy reduction process is effective across all datasets. On Breast, the entropy shows larger oscillations, likely due to the limited number of instances, which amplifies predictive uncertainty. The remaining datasets exhibit a two-stage behavior, with a rapid drop in early iterations as high-uncertainty instances are corrected, followed by a slower decay as the model refines decision boundaries around hard cases. Overall, these trends demonstrate that our method consistently reduces uncertainty in the predictive distribution.

\begin{wrapfigure}{r}{0.46\textwidth}
%\vspace{-12pt}
%\begin{figure}[h]
  \centering
  \includegraphics[width=\linewidth]{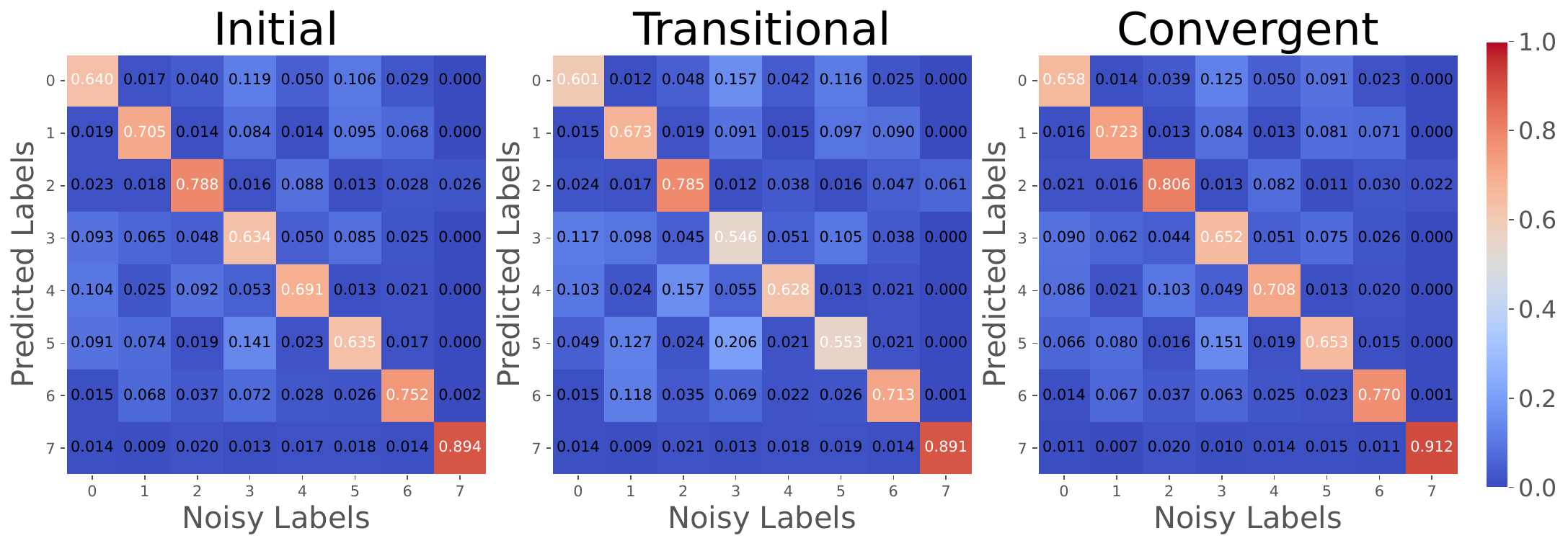}
  \caption{Example of noisy transition process (10\% symmetric noise). We present the initial state, intermediate transition state, and the final convergence of the transition matrix on Blood.}
\label{fig:tm}
%\end{figure}
\vspace{-10pt}
\end{wrapfigure}
\noindent \textbf{Visualization of noisy transition process.} 
In Fig.~\ref{fig:tm}, we visualize the evolution of the class-dependent noise transition (10\% symmetric noise) on Blood as an example. Even though early iterations contain noticeable off-diagonal mass, as the transition proceeds, diagonal entries increase while off-diagonal probabilities are compressed and become sparse, indicating progressive noise suppression. This evolution aligns with our iterative estimation of the transition matrix, which leverages high-confidence predictions to correct the noise distribution.

\section{Related Work}
\label{sec:rewk}
\textbf{Noisy-label learning.} Learning with noisy labels has been widely studied in classical deep learning models ranging from robust losses~\cite{liu2020peer,feng2022ssr,liu2022robust}, sample selection based on memorization dynamics~\cite{han2018co,chen2019understanding,wang2022scalable}, label correction~\cite{tanaka2018joint,ji2021handle,algan2022metalabelnet}, co-training or co-regularization~\cite{wei2020combating,sheng2025ca2c}, and contrastive representation learning~\cite{li2022selective,karim2022unicon,wu2020topological}, under various noise settings, including class-dependent noise~\cite{liu2015classification, patrini2017making, xia2019anchor, yao2020dual, kye2022learning, zhu2022beyond, liu2023identifiability} and instance-dependent noise~\cite{yao2023better, xia2023combating, yang2024dynamic, nguyen2024noisy, li2025learning}.
Within this research field, a central approach models class-conditional label corruption via a noise transition matrix and applies loss correction through forward or backward mapping~\cite{patrini2017making}. Early estimators depend on anchor points~\cite{liu2015classification}, where one class attains near one-hot posteriors, which is fragile in small and highly corrupted datasets. To release such a strong assumption, anchor-free estimators refine the transition iteratively using high-confidence predictions or structural regularizers~\cite{xia2019anchor, yao2020dual, li2021provably, zhang2021learning, zhu2021clusterability, yang2022estimating, cheng2022class}. Despite strong results on large datasets, many of these methods hinge on selecting high-confidence predictions, and most lack provable convergence of the transition refinement itself. Our framework follows the transition-matrix route to adaptively select pseudo-anchors with a supermartingale criterion on predictive entropy that triggers when confidence progresses monotonically, and stabilizes updates with a moving-average rule. This yields a convergence-guaranteed refinement that is well-suited to the model under the small-data setting, such as medical imaging.

\noindent \textbf{Quantum neural networks for medical image classification.}
Recent studies highlight the potential of quantum neural networks (QNNs) in small-data regimes~\cite{jia2019quantum, hur2022quantum, moussa2024hyperparameter}, such as medical imaging~\cite{mathur2021medical, landman2022quantum, rahman2025nqnn}, where obtaining high-quality annotated labels is often a persistent challenge~\cite{wei2023quantum}. Prior studies demonstrate the feasibility of medical understanding tasks through various approaches, including pure quantum circuits~\cite{yousif2024new}, quantum-classical hybrid models~\cite{mathur2021medical, ajlouni2023medical}, and quantum-inspired deep learning models~\cite{ibrahim2024medical, li2026quantum}. However, the systematic handling of label noise, a common artifact in such small data, remains critically under-explored for \qnns~\cite{ju2022improving}. To address this challenge, we propose a supermartingale-based label transition mechanism to improve loss correction, facilitating robust \qnn\ training under a label noise environment.

\section{Conclusion}
\vspace{-0.1cm}
We introduce Supermartingale-based Label Transition (SLT), an anchor-free loss-correction framework for robust \qnn\ training under noisy labels in medical imaging. SLT treats the historical decrease of predictive entropy as a supermartingale process and uses this signal to trigger stable transition-matrix refinement, avoiding reliance on hard-to-find anchor points and providing a theoretical guarantee to steady convergence. Experiments on multiple small-scale medical image datasets demonstrate that SLT improves \qnn-based classification under various label noise settings, consistently outperforming representative baselines.

A main \textbf{limitation} is the computational cost of \qnn\ simulation, a well-known bottleneck in Quantum Machine Learning (QML) rather than a specific drawback of our framework.
\textbf{Future work} will evaluate SLT under more complex real-world noise patterns and across broader medical imaging modalities. More \textbf{broadly}, our framework has the potential to enhance robust medical image analysis, particularly for smaller clinics or healthcare centers in rural areas with limited resources to obtain high-quality medical data.

{
%\small
\bibliographystyle{plainnat}
\bibliography{ref}
}

% Appendix
\newpage
\appendix
In the appendix, we first provide full proof for our theoretical analysis, and further exhibit supplementary experimental settings (e.g., noise design and model architecture) and results (e.g., update strategies). The source code will be released upon acceptance to support reproducibility.

\section{Theoretical Analysis}
\label{app:proof}
\noindent \textbf{Full proof of the theoretical analysis.}
We provide full proof of our theoretical analysis as follows.
\begin{proof}[Proof for Lemma~\ref{lem:spmg}] 
According to Def.~\ref{def:martingale}, a process $S^{(t)}$ is a supermartingale relative to $(\Omega, \mathcal{F}, P)$ if it satisfies Adaptedness, Integrability, and Supermartingale.

\noindent \textbf{Adaptedness.}
We first aim to verify that $S^{(t)}$ is determined based on the information available up to past $t$ iterations. $S^{(t)} = S^{(0)} - \sum_{\tau=1}^{t}\Delta^{(\tau)}$ is a finite subtraction of random variables that are measurable w.r.t. $\sigma\bigl(\Delta^{(1)}, \dots, \Delta^{(t)} \bigr)$. Thus, $S^{(t)}$ is also measurable w.r.t. $\mathcal{F}^{(t)}$, ensuring the adaptedness.

\noindent \textbf{Integrability.}
From Eq.~(\ref{eqn:entropy}), the normalized predictive entropy $s$ lies in $[0,1]$. Consequently, its historical minimum also resides in $[0,1]$. Hence, we obtain $\mathbb{E}[|S^{(t)}|] \leq 1 < \infty$, which guarantees integrability of $\mathbb{E}[|S^{(t)}|]$ for all $t$.

\noindent \textbf{Supermartingale.}
Let $\mathbb{E}_{\mathcal{F}^{(t-1)}}[\cdot] := \mathbb{E}[\ \cdot\!\mid\!\mathcal{F}^{(t-1)}]$. 
Since $S^{(t-1)}$ is $\mathcal{F}^{(t-1)}$-measurable, $\mathbb{E}_{\mathcal{F}^{(t-1)}}[S^{(t-1)}] = S^{(t-1)}$. The derivation is as follows:
{\small
\begin{align*}
    \mathbb{E}_{\mathcal{F}^{(t-1)}}\left[ S^{(t)} \right] 
    &= \mathbb{E}_{\mathcal{F}^{(t-1)}}\left[ S^{(t-1)} - \Delta^{(t)} \right] \\
    &= \mathbb{E}_{\mathcal{F}^{(t-1)}}\left[ S^{(t-1)} \right] - \mathbb{E}_{\mathcal{F}^{(t-1)}}\left[ \Delta^{(t)} \right] \\ 
    &= S^{(t-1)} - \mathbb{E}_{\mathcal{F}^{(t-1)}}\left[ \max(S^{(t-1)}-s^{(t)}, 0) \right] \\ 
    &\leq S^{(t-1)} - \max(\mathbb{E}_{\mathcal{F}^{(t-1)}}\left[ S^{(t-1)}-s^{(t)} \right], 0) \\
    &\leq S^{(t-1)}.
\end{align*}
}

Thus, the supermartingale condition holds true for $\forall$ $t \ge 1$ s.t.
\[
\mathbb{E}\bigl[S^{(t)} \big| \mathcal{F}^{(t-1)}\bigr] \leq S^{(t-1)}, \quad \forall\ t \ge 1.
\]
\end{proof}

\begin{proof}[Proof for \Cref{thm:convergence}]
By Thm.~\ref{thm:fct}, the limit $\lim_{t \to \infty}S^{(t)} = S^{(\infty)}$ exists and is finite. 
Given that $\{S^{(t)}\}$ is a monotonically non-increasing sequence, this indicates that the difference between two adjacent terms must almost surely approach 0,
\[
\lim_{t\to\infty}\Delta^{(t)} = \lim_{t\to\infty} \bigl( S^{(t-1)} - S^{(t)} \bigr) = 0.
\]

By Algo.~\ref{algo:slt}, the condition for updating the transition matrix ${\bf T}$ is $s^{(t)} < S^{(t-1)}$, equivalent to $\Delta^{(t)} > 0$.
As proven, $\lim_{t\to\infty}\Delta^{(t)} = 0$ (a.s.), which implies that the updates of ${\bf T}$ become asymptotically infrequent. 

The above results indicate that for any $\epsilon > 0$, there exists a positive integer $T$ such that $\forall s, t > T$,
\[
|{\bf T}^{(s)} - {\bf T}^{(t)}| < \epsilon,
\]
which implies that the sequence $\{ {\bf T}^{(t)} \}$ is a Cauchy sequence in the space of stochastic transition matrices. Since this space is complete, the sequence $\{ {\bf T}^{(t)} \}$ converges to a fixed-point matrix ${\bf T}^{*}$.
\end{proof}

\noindent \textbf{Analysis of time and space complexity.}
Time complexity per iteration is dominated by the \qnn's forward/backward propagation with parameter size $\mathcal{O}(LNR)$, which may generally take $\mathcal{O}(|X|\!\cdot\!LR\!\cdot\!2^N)$ in classical quantum simulators, while the label correction step requires a dot product between predictive probabilities and the transition matrix, which takes $\mathcal{O}(|X|K^2)$ cost. Besides, computing cross-entropy and predictive entropy takes $\mathcal{O}(|X|K)$, whereas updating the transition matrix takes $\mathcal{O}(|X|)$. Overall, the {\bf time complexity} is $\mathcal{O}(|X|\!\cdot\!LR\!\cdot\!2^N + |X|K^2)$. During training, we store the model parameters $\mathcal{O}(LNR)$, the predictive probabilities $\mathcal{O}(|X|K)$, and the transition matrix $\mathcal{O}(K^2)$. So, the {\bf space complexity} is $\mathcal{O}(LNR + |X|K + K^2)$.

\section{Supplementary Experiments}
\label{app:exp}
\noindent \textbf{Data Statistics.}
We summarize the data statistics of the MedMNIST datasets in Tab.~\ref{tab:dataset_stats}.
%\begin{wraptable}{r}{0.6\textwidth}
%\vspace{-10pt}
\begin{table}[h]
  \centering
  \setlength{\tabcolsep}{1.9pt}
  %\scriptsize
  \caption{Statistics of the MedMNIST datasets used in our experiments. All images are converted to gray scale with $28 \times 28$ size. $|X_{tr}|$, $|X_{val}|$, $|X_{te}|$, $|X|$, and $K$ denote the number of train/validation/test/total instances and classes, respectively.}
  \label{tab:dataset_stats}
  \begin{tabular}{llrrrrr}
    \toprule
    Dataset & Modality & $|X_{tr}|$ & $|X_{val}|$ & $|X_{te}|$ & $|X|$ & $K$ \\ \midrule
    \textbf{Breast} & Breast Ultrasound & 546 & 78 & 156 & 780 & 2 \\
    \textbf{Pneumonia} & Chest X-Ray & 4,708  & 524 & 624 & 5,856 & 2 \\
    \textbf{Retina} & Fundus Camera & 1,080  & 120 & 400 & 1,600 & 5 \\
    \textbf{Derma} & Dermatoscope & 7,007  & 1,003 & 2,005 & 10,015 & 7 \\
    \textbf{Blood} & Microscope & 11,959 & 1,712 & 3,421 & 17,092 & 8 \\ 
    \bottomrule
  \end{tabular}
\end{table}
%\vspace{-10pt}
%\end{wraptable}

%\begin{wraptable}{r}{0.5\textwidth}
%\vspace{-10pt}
\begin{table}[h]
\centering
\small
\setlength{\tabcolsep}{2pt}
\caption{Mapping rubric for the custom-mapping noise. Each class is paired with a clinically or visually confusable counterpart.}
\label{tab:cm_design}
  \begin{tabular}{lll}
  \toprule
  \textbf{Dataset} & \textbf{Original Classes} & \textbf{Noisy-paired Classes} \\
  \midrule
  \multirow{2}{*}{\textbf{Breast}} & 0 (malignant) & 1 (benign) \\
  & 1 (benign) & 0 (malignant) \\
  \midrule
  \multirow{2}{*}{\textbf{Pneumonia}} & 0 (normal) & 1 (pneumonia) \\
  & 1 (pneumonia) & 0 (normal) \\
  \midrule
  \multirow{5}{*}{\textbf{Retina}} & 0 (no DR) & 1 (mild NPDR) \\
  & 1 (mild NPDR) & 0 (no DR) \\
  & 2 (moderate NPDR) & 3 (severe NPDR) \\
  & 3 (severe NPDR) & 2 (moderate NPDR) \\
  & 4 (PDR) & 3 (severe NPDR) \\
  \midrule
  \multirow{7}{*}{\textbf{Derma}} & 0 (AKIEC) & 1 (BCC) \\
  & 1 (BCC) & 0 (AKIEC) \\
  & 2 (BKL) & 5 (NV) \\
  & 3 (DF) & 2 (BKL) \\
  & 4 (MEL) & 5 (NV) \\
  & 5 (NV) & 4 (MEL) \\
  & 6 (VASC) & 1 (BCC) \\
  \midrule
  \multirow{8}{*}{\textbf{Blood}}
 & 0 (basophil) & 1 (eosinophil) \\
 & 1 (eosinophil) & 0 (basophil) \\
 & 2 (erythroblast) & 3 (immature granulocytes) \\
 & 3 (immature granulocytes) & 6 (neutrophil) \\
 & 4 (lymphocyte) & 5 (monocyte) \\
 & 5 (monocyte) & 4 (lymphocyte) \\
 & 6 (neutrophil) & 3 (immature granulocytes) \\
 & 7 (platelet) & 2 (erythroblast) \\
 \bottomrule
 \end{tabular}
\end{table}
%\vspace{-10pt}
%\end{wraptable}
\noindent \textbf{Noise implementation and design.}
We follow the setting in~\textbf{Forward}~\cite{patrini2017making} to construct uniform (UN) and pair-wise (Cyclic-flipping, UF) noise. Furthermore, we design a new pair-wise noise (Custom-mapping, CM) and justify the design, whose mapping rubric is provided in Tab.~\ref{tab:cm_design}. For the binary-class datasets (\textbf{Breast} and \textbf{Pneumonia}), we use a standard pair-wise mapping between the two classes. For the complex multi-class datasets, the pairings are designed to simulate realistic, clinically-plausible diagnostic errors.

For \textbf{Retina}, we leverage the ordinal and progressive nature of diabetic retinopathy (DR) classification. The dataset comprises five clinically defined stages of DR: no apparent retinopathy, mild non-proliferative DR (NPDR), moderate NPDR, severe NPDR, and proliferative DR (PDR). Misdiagnoses in real-world settings often occur between adjacent stages due to overlapping retinal features such as microaneurysms, hemorrhages, and neovascularization, especially under suboptimal imaging conditions. To simulate such realistic diagnostic ambiguities, we construct a class-conditional pair noise schema wherein each DR stage is paired with its immediate neighbor. Specifically, ``no DR'' is paired with ``mild NPDR'', ``moderate NPDR'' with ``severe NPDR'', and ``severe NPDR'' with ``PDR''. This setup reflects both clinical practice and morphological continuity.

For \textbf{Derma}, we construct a pairwise noise mapping based on dermatological similarity, visual confusion, and well-documented misdiagnosis patterns in dermoscopy. Each class is mapped to a single ``confusable'' counterpart, capturing asymmetric mislabeling commonly observed in medical practice. For instance, actinic keratoses and intraepithelial carcinoma (AKIEC) are frequently mistaken for basal cell carcinoma (BCC) due to overlapping visual features such as keratinization and pigmentation. Similarly, melanoma (MEL) and melanocytic nevi (NV) represent a classic diagnostic challenge, as their visual boundaries and color distributions often overlap, leading to bidirectional confusion. Benign keratosis-like lesions (BKL) can appear similar to nevi, while dermatofibroma (DF) may be misinterpreted as BKL due to shared dermal structures. Vascular lesions (VASC), characterized by reddish nodules, may occasionally mimic the appearance of BCC in dermoscopic images.

For \textbf{Blood}, each cell type is paired with another class that it is most likely to be confused with due to visual resemblance under light microscopy, particularly in low-resolution or poorly stained images. For instance, basophils and eosinophils (classes 0 and 1) are both granulocytes with similar nuclear granulation, making them difficult to differentiate. Erythroblasts (class 2), which are immature red blood cells, share size and chromatin features with immature granulocytes (class 3), especially in early developmental stages. Likewise, lymphocytes and monocytes (classes 4 and 5) are frequently misclassified due to overlapping sizes and nuclear shapes. Finally, platelets (class 7), being small and dense, may be mistaken for erythroblastic fragments under certain imaging conditions.

\noindent \textbf{Implementation details and hyperparameters for baselines and competing methods.}
We follow the overall setting in~\textbf{Forward}~\cite{patrini2017making} to build the initial transition matrix. The key difference is that we train a two-layer shallow neural network on the training data without relying on anchor points.
To ensure a fair comparison, we implement all baselines and competing methods using the same backbone \qnn\ and optimizer with an identical batch size (128) and learning rate (0.01). Specifically, for \textbf{VolMinNet}~\cite{li2021provably}, we set the regularization coefficient to $10^{-3}$. For \textbf{TVR}~\cite{zhang2021learning}, we set the regularization coefficient $\gamma = 0.1$ and the beta distribution parameters to $[0.999, 0.01]$. For \textbf{BLTM}~\cite{yang2022estimating}, the upper bound threshold $\rho_{\max}$ is set to $0.6$. For \textbf{CCR}~\cite{cheng2022class}, the weight of the consistency regularization term is set to $\lambda=0.3$.
For competing methods validated on CheXpert, we conducted a grid search to fine-tune the key hyperparameters and, for the rest, followed the same settings as those in the original papers.
For \textbf{DivideMix}~\cite{lidividemix}, we limited the warm-up period to 10 epochs to accommodate the smaller dataset, while fixing the Beta parameter, probability threshold, and sharpening temperature all at 0.5. In \textbf{ELR}~\cite{liu2020early}, we applied a regularization weight of 3.0 and set the temporal ensembling momentum to 0.7. The \textbf{VisualCheXbert}~\cite{jain2021visualchexbert} was set up to utilize both label smoothing and KL-divergence. Regarding \textbf{BoMD}~\cite{chen2023bomd}, we adopted a projection dimension of 128, maintained a memory bank queue size of 8192, and set the temperature parameter $\tau$ to 0.07. Finally, for \textbf{FixMatch}~\cite{ihler2024distribution}, we established a high confidence threshold of 0.95, assigned an unlabeled loss weight of 1.0, and used an unlabeled data ratio of 5.

\begin{wrapfigure}{r}{0.5\textwidth}
\vspace{-10pt}
%\begin{figure}[h]
\centering
  \includegraphics[width=\linewidth]{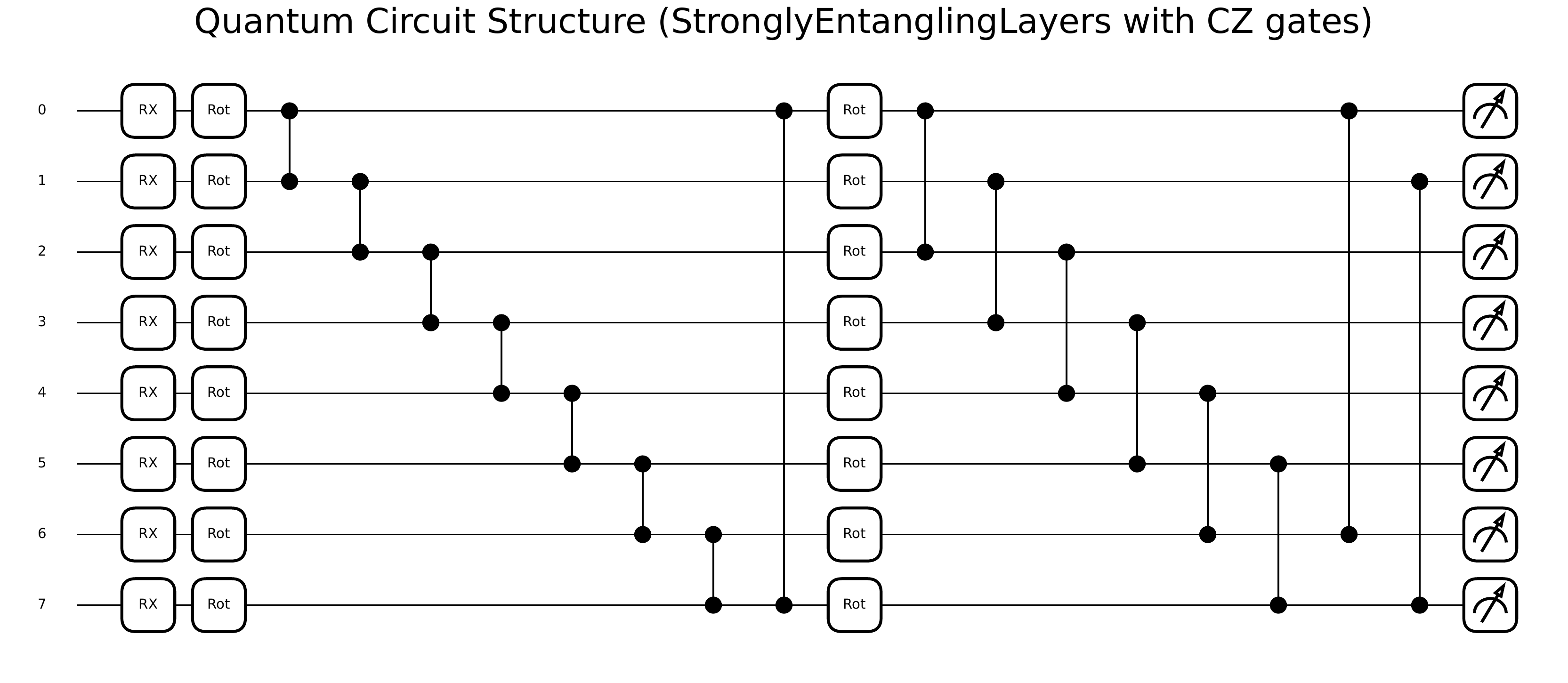}
  \caption{Model architecture of our backbone quantum circuit (with 8 qubits, 2 layers, and 3 rotation gates) using PennyLane's `\textit{qml.StronglyEntanglingLayers}'.}
\label{fig:vqc}
%\end{figure}
\vspace{-10pt}
\end{wrapfigure}
\noindent \textbf{Model architecture of our backbone \qnns.}
In this study, we evaluate our framework using a \qnn\ as our backbone model. Unlike a standalone quantum circuit, our model incorporates classical pre-processing and post-processing layers to enhance learnability. First, high-dimensional classical inputs are compressed via a fully connected layer and normalized using a hyperbolic tangent ($\tanh$) activation function. This step ensures that input values remain bounded, preventing the rotation angles in the subsequent quantum embedding from diverging. The processed features are then encoded into quantum states via angle encoding using $R_X$ rotations. The variational ansatz employs \textit{StronglyEntanglingLayers}, which consists of parameterized single-qubit rotations and entangling $CZ$ gates (controlled-phase) rather than standard CNOT gates. Finally, the expectation values of the Pauli-Z operators are measured across all qubits and mapped to the final output classes through a second classical fully connected layer.

\begin{table}[h]
\centering
\caption{Comparison with baselines across five datasets under 10\% instance-dependent noise (IDN) by mean(std) of F1 (\%) over five runs. The best results are highlighted in bold.}
\label{tab:idn}
\resizebox{\linewidth}{!}{
%\footnotesize
%\setlength{\tabcolsep}{1pt}
\begin{tabular}{lccccc}
\toprule
Models & Breast & Pneumonia & Retina & Derma & Blood \\
\midrule
DISC & 65.71 (2.11) & 83.36 (2.42) & 27.08 (2.27) & 25.94 (1.73) & 66.88 (1.08) \\
ProMix & 66.34 (2.03) & 84.19 (2.08) & 28.21 (2.01) & 26.83 (1.61) & 67.57 (0.98) \\
InstanceGM & 65.88 (1.94) & 83.74 (1.97) & 28.76 (2.12) & 27.19 (1.56) & 67.96 (1.03) \\
$\pi$-LR & 66.02 (1.86) & 83.87 (1.89) & 28.94 (2.05) & 27.37 (1.42) & 68.06 (0.95) \\
Ours & \textbf{69.07 (2.32)} & \textbf{86.15 (1.72)} & \textbf{33.53 (2.80)} & \textbf{28.67 (1.50)} & \textbf{69.25 (0.95)} \\
\bottomrule
\end{tabular}
}
\end{table}
\noindent \textbf{Validation under instance-dependent noise (IDN).}
To validate the effectiveness of our method under IDN, we follow the benchmark setup~\cite{ma2025benchmarking}, which in turn follows~\cite{xia2020part}, to construct instance-dependent noise and compare our method with the following leading baselines in the 10\% IDN setting. We finetune the hyperparameters on validation sets and choose the setting of 0.4 $\eta$, 0.5 del, and 20 pat. For fair comparison, we use the same backbone model and optimization hyperparameters.
\textbf{DISC}~\cite{li2023disc} dynamically selects reliable instances and corrects suspicious labels during training in an instance-specific manner.
\textbf{ProMix}~\cite{xiao2023promix} combats label noise by maximizing the utility of clean samples while reducing the harmful effect of corrupted ones. We set the threshold as 0.9.
\textbf{InstanceGM}~\cite{garg2023instance} models instance-dependent label corruption with graphical modeling. We set the threshold as 0.5, and $\alpha=4$, $\lambda_u=22$, $T=0.5$.
\textbf{$\pi$-LR}~\cite{he2025probabilistic} models each sample's confusing probability and refines labels probabilistically to handle IDN. We set both elr and consistency ratio as 0.75, $\lambda=0.3$.
As shown in Tab.~\ref{tab:idn}, the results verify that our method achieves robust performance under IDN settings and outperforms these baselines.

\begin{wraptable}{r}{0.5\textwidth}
\vspace{-13pt}
%\begin{table}[h]
\centering
\caption{Comparison of competing methods on real-world noisy labeled dataset, CheXpert. We report the mean AUC (\%) and F1 score (\%) $\pm$ standard deviation over five independent runs. The best results are highlighted in \textbf{bold}.}
\label{tab:CheXpert}
\begin{tabular}{lcc}
\toprule
\textbf{Methods} & \textbf{AUC (\%)} & \textbf{F1 (\%)} \\
\midrule
\textbf{CE (U-Zero)} & 87.10 (0.15) & 64.20 (0.05) \\
\textbf{DivideMix} & 87.88 (0.22) & 65.85 (0.15) \\
\textbf{VisualCheXbert} & 88.11 (0.18) & 66.22 (0.04) \\
\textbf{ELR} & 87.52 (0.20) & 65.08 (0.09) \\
\textbf{BoMD} & 88.47 (0.12) & 67.55 (0.03) \\
\textbf{FixMatch} & 88.03 (0.25) & 66.59 (0.16) \\
\textbf{Ours} & \textbf{89.85 (0.10)} & \textbf{68.94 (0.02)} \\
\bottomrule
\end{tabular}
%\end{table}
\vspace{-8pt}
\end{wraptable}
\noindent \textbf{Validation on CheXpert, a real-world, noisy labeled dataset.}
We evaluate our proposed method on the real-world, noisy labeled chest radiograph dataset, CheXpert~\cite{irvin2019chexpert}, which consists of 224,316 images collected from 65,240 patients at Stanford Hospital. Labels for the training set were automatically extracted from radiology reports using a natural language processing (NLP) labeler, resulting in four categories: positive, negative, uncertain, and unmentioned. Following standard evaluation protocols, we focus on the five competition pathologies: Atelectasis, Cardiomegaly, Consolidation, Edema, and Pleural Effusion. For performance assessment, we utilize the official test set, comprising 500 studies annotated by the consensus of five board-certified radiologists, which serves as the clean ground truth. To simulate the data scarcity, we employ a stratified sampling strategy to sample 10\% of the training set. Also, we apply the U-Zeros approach to assign the uncertain labels. We finetune the hyperparameters on validation sets and choose the setting of 0.4 $\eta$, 0.5 del, and 15 pat. In experiments, we compute the average AUC and F1 of the above five competition pathologies, and repeat experiments five times. The results are reported in mean (std).
We compare our method with both generic noisy-label learning baselines and domain-specific medical imaging methods, which are widely applied on the CheXpert dataset.
\textbf{DivideMix}~\cite{lidividemix} is a semi-supervised framework that uses a Gaussian Mixture Model to dynamically divide samples into labeled (clean) and unlabeled (noisy) sets for co-training.
\textbf{VisualCheXbert}~\cite{jain2021visualchexbert} is a domain-specific method that utilizes a BERT-based text encoder to extract high-quality soft labels from radiology reports to supervise the image classifier.
\textbf{ELR}~\cite{liu2020early} is a regularization technique that leverages the ``early learning'' phenomenon to prevent deep networks from memorizing false labels by enforcing prediction consistency.
\textbf{BoMD}~\cite{chen2023bomd} is a state-of-the-art medical method that employs text-driven multi-label descriptors to guide the image model in learning distinct pathological features under noise.
\textbf{FixMatch}~\cite{ihler2024distribution} is a semi-supervised learning method that enforces consistency between the model's predictions on weakly-augmented and strongly-augmented versions of the same image.

As presented in Tab.~\ref{tab:CheXpert}, our proposed method achieves a new state-of-the-art performance on the CheXpert dataset, outperforming both generic noisy-label learning baselines and domain-specific medical imaging methods. Specifically, our method reaches an AUC of 89.85\% and an F1 score of 68.94\%, surpassing the previous strongest competitor, BoMD. It is worth noting that while general noisy-label methods like DivideMix and ELR provide improvements over the standard Cross-Entropy baseline (87.88\% and 87.52\% vs. 87.10\% in AUC), their gains are limited compared to our approach. This suggests that general noise-label learning strategies developed for natural images (e.g., CIFAR) may struggle under data scarcity. The results validate that our methods work effectively on real-world, noisy labeled datasets.

\begin{wraptable}{r}{0.3\textwidth}
\vspace{-12pt}
%\begin{table}[h]
\centering
\small
\caption{Model size analysis.}
\label{tab:msize}
  \begin{tabular}{lcc}
    \toprule
    \textbf{Model} & \bm{$|\theta|$} & \textbf{Scale} \\
    \midrule
    SNN & $2{,}146$ & $3.66\times$ \\
    DNN & $23{,}266$ & $39.70\times$ \\
    \qnn\ & $586$ & $1\times$ \\
    \bottomrule
  \end{tabular}
%\end{table}
\vspace{-10pt}
\end{wraptable}
\noindent \textbf{Analysis of trainable parameters.}
We compare the number of trainable parameters among the backbone models, SNN, DNN, and \qnn. As shown in Tab.~\ref{tab:msize}, \qnn\ contains only $586$ trainable parameters when $K=2$ (the scale is normalized by \qnn). In contrast, the shallow classical model contains $2{,}146$ parameters, which is about $3.66\times$ larger than ours, while the deep classical model with 10 residual blocks contains $23{,}266$ parameters, about $39.70\times$ larger. This substantial reduction in model size suggests that our quantum model provides a highly parameter-efficient alternative, achieving a compact architecture while retaining expressive modeling capacity.

\begin{wrapfigure}{r}{0.52\textwidth}
\vspace{-10pt}
%\begin{figure}[h]
  \centering
  \includegraphics[width=\linewidth]{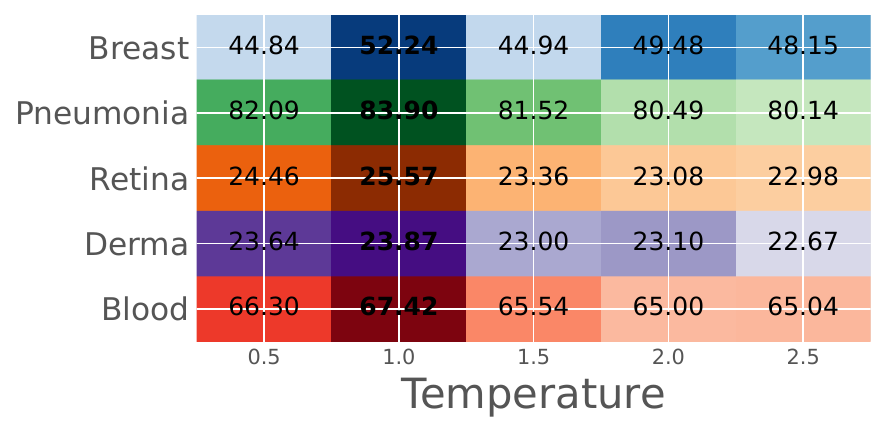}
  \caption{An Analysis of the impact of smoothness on \qnns\ across five datasets under 50\% uniform noise. We apply temperature scaling to the output of \qnns\ before utilizing SLT to simulate different scales of smoothness and report the F1 score (\%).}
\label{fig:temp}
%\end{figure}
\vspace{-10pt}
\end{wrapfigure}
\noindent \textbf{Analysis of smoothness in \qnns.}
To further examine how the smoothness affects \qnn, we analyze the F1 scores of the variant ``QNN+TS+SLT'' under 50\% uniform noise across different temperature values. In Fig.~\ref{fig:temp}, we show a row-wise normalized heatmap, where each row corresponds to a dataset and each column corresponds to a temperature value; within each row, darker cells indicate better F1 for that dataset. A clear and consistent trend emerges across all five datasets: our method achieves the best F1 at Temp = 1.0, namely, when no temperature scaling is applied to the \qnn. This observation suggests that the intrinsic smoothness of \qnns\ already mitigates the adverse effects of overconfident predictions, making additional output-scale adjustment unnecessary. In turn, this provides further evidence for the advantage of using \qnns\ as the backbone model in our framework.

\begin{figure*}[h]
 \centering
 \includegraphics[width=0.99\textwidth]{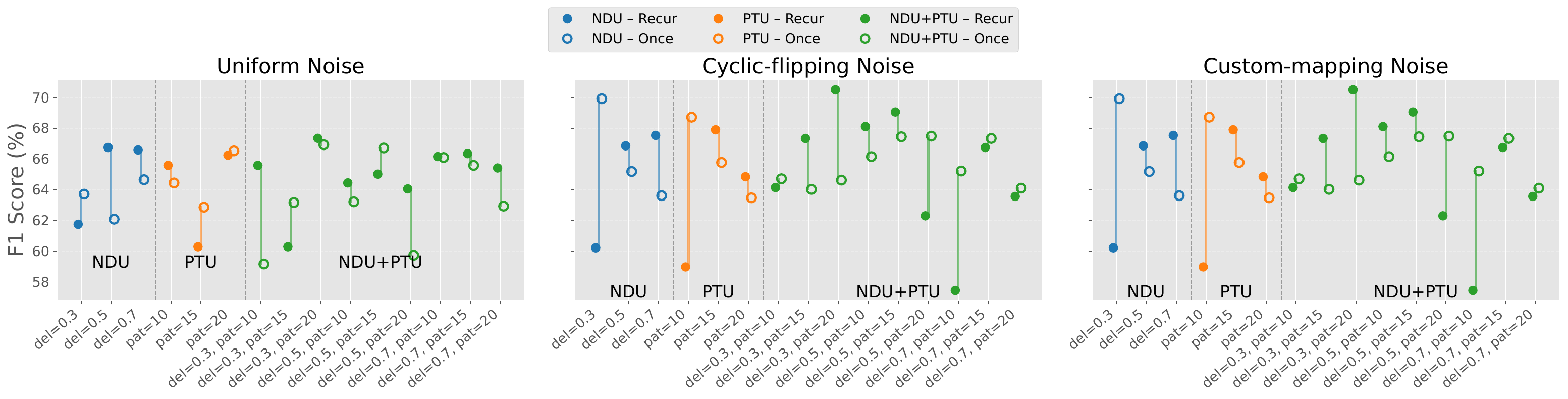}
\caption{Analysis of update strategies (NDU, PTU, and NDU+PTU) under two scenarios (Recur and Once) for the transition matrix under three noise types (Uniform, Cyclic-flipping, and Custom-mapping) with a 10\% noise ratio on Breast. ``NDU'' delays the transition matrix update until a specified fraction (``del'') of total training iterations is completed. ``PTU'' triggers an update only after performance stagnates (e.g., no F1 increase or entropy decrease) for ``pat'' consecutive rounds. ``NDU+PTU'' signifies that both strategies are implemented concurrently. ``Recur'' and ``Once'' denote whether or not the above strategies are applied repeatedly whenever their conditions are met. The vertical line indicates the gap in the F1 score under two scenarios.}
\label{fig:update}
\end{figure*}

\begin{table*}[t]
\centering
\newcommand{\std}[1]{{\footnotesize\ (#1)}}
\caption{Comparison with baselines (CE, Forward) and competing methods across Derma and Blood, under three noise settings (UN: Uniform, CF: Cyclic-flipping, CM: Custom-mapping) at noise levels 10\%/30\%/50\%. We report mean(std) of F1~(\%) over five runs.}
\label{tab:compare2}
\footnotesize
\setlength{\tabcolsep}{1.5pt}
\begin{tabular}{c|c|ccc|ccc}
\toprule
\multirow{2}{*}{\textbf{Noise}} & \multirow{2}{*}{\textbf{Models}}
& \multicolumn{3}{c|}{\textbf{Derma}}
& \multicolumn{3}{c}{\textbf{Blood}} \\
\cmidrule(lr){3-5}\cmidrule(l){6-8}
 & & 10\% & 30\% & 50\% & 10\% & 30\% & 50\% \\
\midrule
\multirow{9}{*}{\textbf{UN}}
 & CE & 24.01\std{1.25} & 21.08\std{2.32} & 19.82\std{2.98} & 68.78\std{0.77} & 67.24\std{0.61} & 65.29\std{0.58}\\
 & Forward & 27.13\std{1.33} & 25.93\std{1.75} & 22.94\std{1.80} & 69.03\std{0.93} & 68.56\std{0.70} & 66.74\std{0.65}\\
 & T-Revision & 16.44\std{3.11} & 16.69\std{4.37} & 14.61\std{3.83} & 64.37\std{5.97} & 59.91\std{6.67} & 57.91\std{6.68}\\
 & Dual-T & 27.28\std{0.74} & 25.03\std{1.67} & 22.75\std{1.79} & 69.13\std{0.88} & 63.78\std{3.90} & 59.15\std{3.86}\\
 & VolMinNet & 26.09\std{1.94} & 24.89\std{2.16} & 21.94\std{2.57} & 68.47\std{1.89} & 62.41\std{4.62} & 57.23\std{4.39}\\
 & TVR & 17.87\std{2.75} & 16.06\std{4.47} & 11.66\std{3.67} & 53.18\std{8.08} & 46.83\std{8.99} & 39.42\std{9.56}\\
 & BLTM & 19.56\std{1.77} & 18.59\std{1.57} & 16.90\std{3.21} & 68.61\std{1.38} & 64.98\std{2.61} & 57.51\std{3.65}\\
 & CCR & 25.32\std{0.74} & 25.06\std{0.43} & 22.86\std{1.99} & 68.62\std{0.67} & 67.87\std{1.58} & 66.82\std{1.09}\\
 & Ours & \textbf{29.24\std{1.71}} & \textbf{26.95\std{1.34}} & \textbf{23.87\std{1.56}} & \textbf{69.88\std{0.63}} & \textbf{69.14\std{1.26}} & \textbf{67.42\std{1.51}}\\
\midrule
\multirow{9}{*}{\textbf{CF}}
 & CE & 21.29\std{2.34} & 19.77\std{2.02} & 15.59\std{3.86} & 68.10\std{0.89} & 66.10\std{1.15} & 38.97\std{4.12}\\
 & Forward & 26.56\std{1.67} & 24.37\std{1.91} & 15.90\std{2.38} & 70.17\std{0.97} & 68.49\std{1.39} & 46.38\std{3.96}\\
 & T-Revision & 21.52\std{3.48} & 18.88\std{4.02} & 14.39\std{4.57} & 64.24\std{5.18} & 54.76\std{4.49} & 44.53\std{4.10}\\
 & Dual-T & 26.66\std{1.55} & 24.75\std{1.81} & 15.40\std{1.38} & 70.14\std{0.46} & 67.63\std{3.65} & 50.74\std{6.42}\\
 & VolMinNet & 26.96\std{1.89} & 23.73\std{2.62} & 14.87\std{0.78} & 69.57\std{0.19} & 68.27\std{3.61} & 48.98\std{5.01}\\
 & TVR & 20.46\std{4.13} & 15.32\std{2.78} & 13.33\std{4.34} & 59.74\std{8.08} & 55.01\std{6.58} & 42.58\std{5.72}\\
 & BLTM & 22.62\std{2.73} & 19.68\std{2.19} & 16.19\std{2.15} & 67.71\std{0.85} & 58.96\std{7.47} & 42.93\std{5.24}\\
 & CCR & 25.00\std{1.03} & 23.94\std{1.62} & 16.82\std{2.64} & 69.09\std{1.38} & 68.30\std{2.84} & 48.41\std{5.33}\\
 & Ours & \textbf{27.69\std{1.52}} & \textbf{25.89\std{1.33}} & \textbf{18.04\std{2.12}} & \textbf{71.81\std{0.95}} & \textbf{69.52\std{0.86}} & \textbf{51.69\std{3.47}}\\
\midrule
\multirow{9}{*}{\textbf{CM}}
 & CE & 20.44\std{2.58} & 19.29\std{2.40} & 15.14\std{3.97} & 69.84\std{0.89} & 66.58\std{0.77} & 40.06\std{2.24}\\
 & Forward & 26.06\std{1.26} & 22.72\std{1.35} & 16.84\std{3.47} & 69.92\std{0.70} & 69.37\std{0.28} & 49.80\std{1.79}\\
 & T-Revision & 20.55\std{2.70} & 18.74\std{3.71} & 16.15\std{3.16} & 66.99\std{3.23} & 54.70\std{1.43} & 40.67\std{4.41}\\
 & Dual-T & 25.06\std{1.56} & 21.72\std{1.21} & 15.93\std{2.75} & 69.03\std{0.68} & 65.02\std{5.76} & 39.02\std{7.78}\\
 & VolMinNet & 25.21\std{0.03} & 20.94\std{2.27} & 14.34\std{3.91} & 69.65\std{1.69} & 63.96\std{4.59} & 40.78\std{5.27}\\
 & TVR & 20.25\std{3.98} & 16.09\std{4.36} & 12.83\std{4.40} & 57.20\std{4.07} & 55.19\std{5.14} & 44.85\std{6.52}\\
 & BLTM & 21.24\std{1.33} & 18.70\std{2.78} & 17.86\std{1.67} & 69.04\std{0.98} & 57.77\std{7.11} & 46.59\std{6.18}\\
 & CCR & 24.86\std{0.86} & 23.10\std{1.86} & 17.25\std{2.84} & 69.67\std{0.68} & 68.27\std{0.69} & 50.34\std{5.54}\\
 & Ours & \textbf{27.51\std{1.92}} & \textbf{23.57\std{1.68}} & \textbf{18.12\std{3.15}} & \textbf{70.56\std{0.97}} & \textbf{69.78\std{0.53}} & \textbf{51.30\std{3.68}}\\
\bottomrule
\end{tabular}
\end{table*}

\begin{wraptable}{r}{0.5\textwidth}
%\vspace{-10pt}
%\begin{table}[h]
\centering
\small
\setlength{\tabcolsep}{3pt}
\caption{Optimal hyperparameters, $\eta$, del, and pat, per dataset and noise type. `UN', `CF', and `CM' denote the noise types of Uniform, Cyclic-flipping, and Custom-mapping, respectively. We adopt the update mode ``NDU+PTU'' for all datasets.}
\label{tab:hyperparams}
 \begin{tabular}{l|ccc|ccc|ccc}
 \toprule
  & \multicolumn{3}{c|}{\textbf{UN}} & \multicolumn{3}{c|}{\textbf{CF}} & \multicolumn{3}{c}{\textbf{CM}} \\
 \cmidrule(lr){2-10}
 & $\eta$ & del & pat & $\eta$ & del & pat & $\eta$ & del & pat \\
 \midrule
 \textbf{Breast}    & 0.6 & 0.3 & 20 & 0.9 & 0.3 & 20 & 0.9 & 0.3 & 20 \\
 \textbf{Pneumonia} & 0.1 & 0.7 & 20 & 1.0 & 0.3 & 20 & 1.0 & 0.3 & 20 \\
 \textbf{Retina}    & 0.1 & 0.3 & 20 & 0.2 & 0.3 & 20 & 0.1 & 0.7 & 15 \\
 \textbf{Derma}     & 1.0 & 0.7 & 10 & 0.3 & 0.3 & 10 & 0.4 & 0.7 & 20 \\
 \textbf{Blood}     & 0.4 & 0.5 & 15 & 0.3 & 0.5 & 15 & 0.4 & 0.7 & 20 \\
 \bottomrule
\end{tabular}
%\end{table}
\vspace{-8pt}
\end{wraptable}
\noindent \textbf{Analysis of update strategies for the transition matrix.} Because frequent updates can destabilize the transition, we examine how different update strategies affect model performance on validation sets. We report results on the Breast dataset and present the remaining results in Fig.~\ref{fig:update2}. As shown in Fig.~\ref{fig:update}, repeatedly executing the strategy (``Recur'') clearly outperforms a single execution (``Once''). The combined strategy ``NDU+PTU'' is the most robust, yielding the best F1 with smaller variance across hyperparameters. Based on these observations, we adopt ``NDU+PTU'' and apply it recurrently during training. 
Overall, based on the analysis of $\eta$, del, and pat, we observe that while SLT benefits from parameter tuning to achieve higher F1 scores, the performance degradation with suboptimal parameters is mild. This demonstrates the framework's robustness to variations in hyperparameter settings.
We summarize the optimal hyperparameters in Tab.~\ref{tab:hyperparams}.

\noindent \textbf{Discussion on failure cases in trade-off analysis.}
Loss-correction approaches depend heavily on an accurately estimated initial transition matrix. Specifically, this reliance necessitates that the backbone model's predictions adhere to a closed-set assumption. In our experiments, we observed that employing an insufficient number of qubits leads to failure in constructing the initial transition matrix. It is important to note, however, that this is not a unique limitation inherent to our proposed framework.

\noindent \textbf{Additional Comparison.}
We present the rest of the comparison results in Tab.~\ref{tab:compare2}.

\noindent \textbf{Hardware and software.}
Our experimental environment consisted of an Ubuntu 22.04.3 LTS system powered by an Intel Xeon w5-3433 CPU and a 48GB NVIDIA RTX A6000 GPU. The software stack included Python 3.11.8, PyTorch 2.2.2, and Pennylane 0.35.1.
The time required for the model training (single run) ranges from 25s (smallest dataset) to 582s (largest dataset).
% The experiment is conducted on a server with the following settings:
% \begin{itemize}%[itemsep=-1mm]
%   \item Operating System: Ubuntu 22.04.3 LTS
%   \item CPU: Intel Xeon w5-3433 @ 4.20 GHz
%   \item GPU: NVIDIA RTX A6000 48GB
%   \item Software: Python 3.11.8, PyTorch 2.2.2, Pennylane 0.35.1.
% \end{itemize}

\noindent \textbf{LLM usage.}
We employ AI tools, e.g., LLMs, to polish the clarity.

\begin{figure*}[h]
  \begin{subfigure}{1.0\linewidth}
    \centering
    \includegraphics[width=0.99\textwidth]{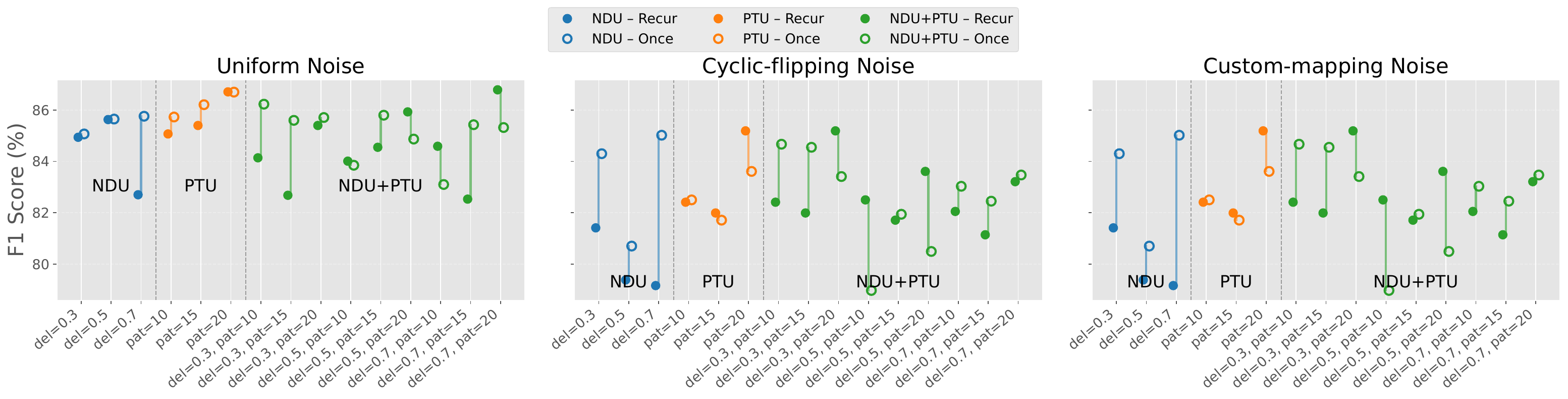}
    \includegraphics[width=0.99\textwidth]{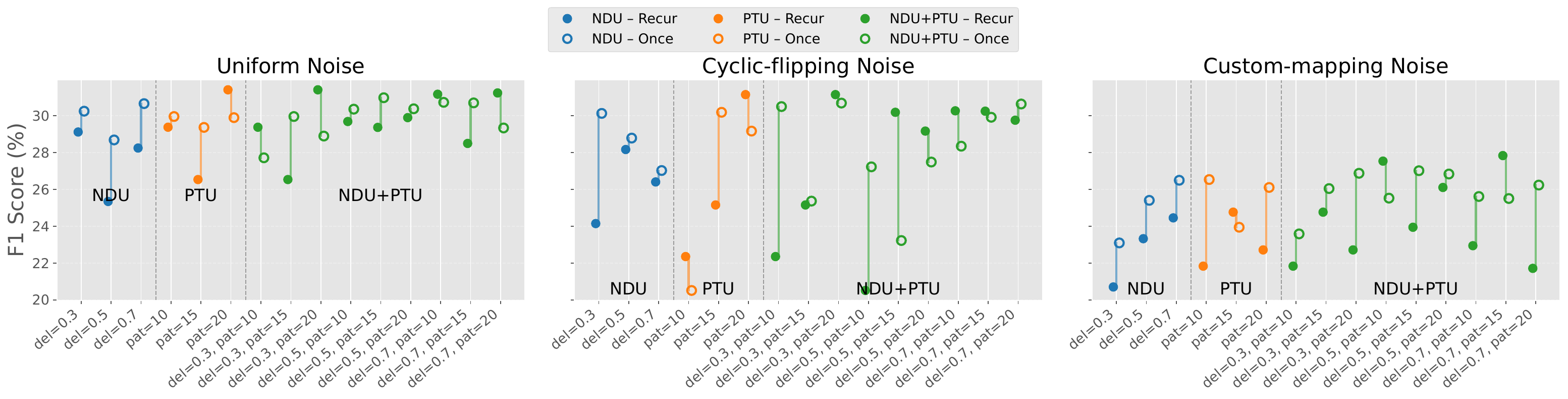}
    \includegraphics[width=0.99\textwidth]{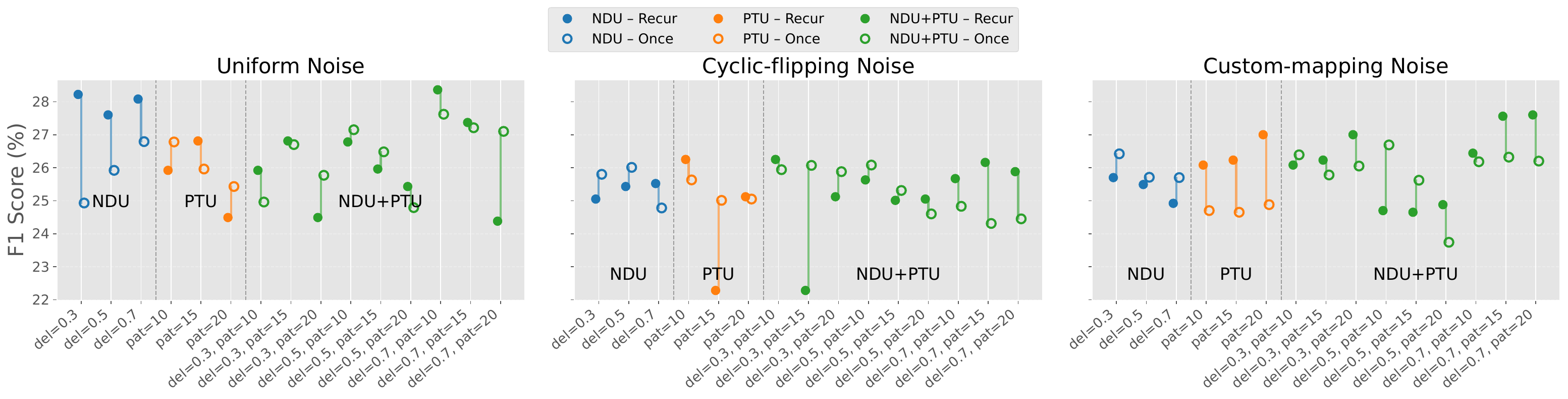}
    \includegraphics[width=0.99\textwidth]{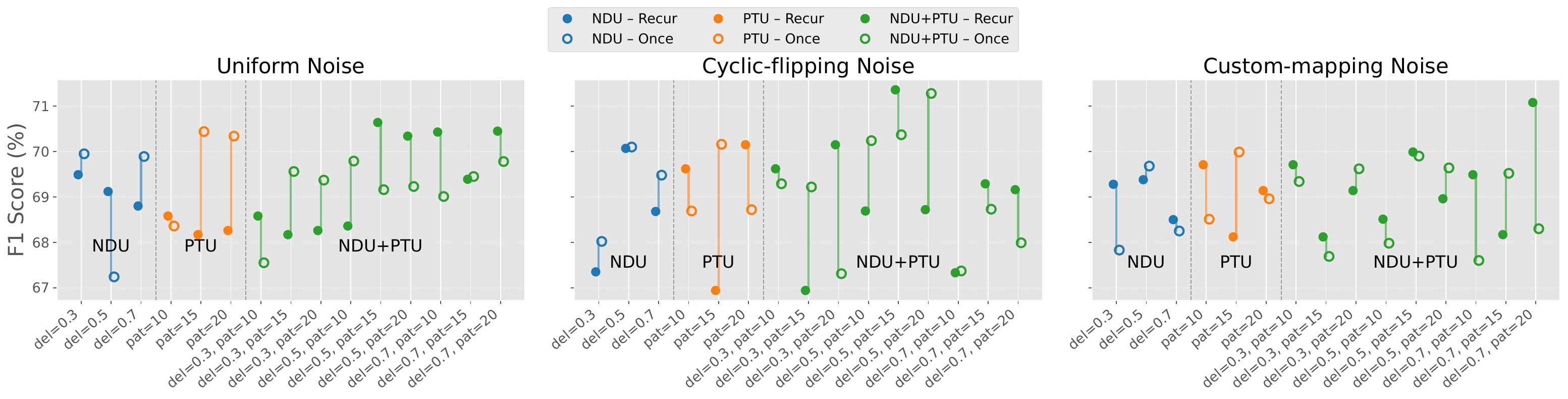}
  \end{subfigure}
\caption{Analysis of update strategies (NDU, PTU, and NDU+PTU) under two scenarios (Recur and Once) for the transition matrix under three noise types (Uniform, Cyclic-flipping, and Custom-mapping) across four datasets in order (Pneumonia, Retina, Derma, and Blood). ``NDU'' delays the transition matrix update until a specified fraction (``del'') of total training iterations is completed. ``PTU'' triggers an update only after performance stagnates (e.g., no F1 increase or entropy decrease) for ``pat'' consecutive rounds. ``NDU+PTU'' signifies that both strategies are implemented concurrently. ``Recur'' and ``Once'' denote whether or not the above strategies are applied repeatedly whenever their conditions are met. The vertical line indicates the gap in the F1 score under two scenarios.}
\label{fig:update2}
\end{figure*}

\end{document}